%% file: main.tex
\documentclass{article}

\usepackage{microtype}
\usepackage{graphicx}
\usepackage{subfigure}
\usepackage{booktabs} % for professional tables
\usepackage{hyperref}

\usepackage[accepted]{icml2023}

\usepackage{amsmath}
\usepackage{amssymb}
\usepackage{mathtools}
\usepackage{amsthm}

\usepackage[capitalize,noabbrev]{cleveref}

\usepackage{multirow}
\usepackage{array}
\usepackage{makecell}
\usepackage{bbm}
\usepackage{rotating}

\theoremstyle{plain}
\newtheorem{theorem}{Theorem}[section]

\newtheorem{lemma}[theorem]{Lemma}

\theoremstyle{definition}
\newtheorem{definition}[theorem]{Definition}

\newtheorem{example}[theorem]{Example}
\newtheorem{axiom}[theorem]{Axiom}
\crefname{axiom}{Axiom}{Axioms}
\theoremstyle{remark}

\newcommand{\A}{\mathcal{A}}

\newcommand{\R}{\mathbb{R}}

\newcommand{\E}{\mathbb{E}}
\newcommand{\Emath}{\mathop{\mathbb{E}}}
\newcommand{\st}{\mathbin{|}} % binary relation
\DeclareMathOperator\erf{erf}
\DeclareMathOperator{\cost}{cost}
\DeclareMathOperator{\revenue}{revenue}
\newcommand{\ind}{\mathbbm{1}}

\newcommand{\prefgt}[1]{\mathrel{\succ_{#1}}}
\newcommand{\prefgteq}[1]{\mathrel{\succeq_{#1}}}
\newcommand{\prefeq}[1]{\mathrel{\simeq_{#1}}}

\icmltitlerunning{Formalizing Preferences Over Runtime Distributions}

\begin{document}

\twocolumn[
\icmltitle{Formalizing Preferences Over Runtime Distributions}
\icmlsetsymbol{equal}{*}

\begin{icmlauthorlist}
\icmlauthor{Devon R. Graham}{ubc}
\icmlauthor{Kevin Leyton-Brown}{ubc}
\icmlauthor{Tim Roughgarden}{columbia,a16z}
\end{icmlauthorlist}

\icmlaffiliation{ubc}{Department of Computer Science, University of British Columbia, Vancouver, BC}
\icmlaffiliation{columbia}{Department of Computer Science, Columbia University, New York, New York}
\icmlaffiliation{a16z}{a16z crypto}

\icmlcorrespondingauthor{Devon R. Graham}{drgraham@cs.ubc.ca}

\icmlkeywords{Algorithm configuration, Algorithm selection, Expected utility, Optimization}

\vskip 0.3in
]

\printAffiliationsAndNotice{}  % leave blank if no need to mention equal contribution

\begin{abstract}
When trying to solve a computational problem, we are often faced with a choice between algorithms that are guaranteed to return the right answer but differ in their runtime distributions (e.g., SAT solvers, sorting algorithms). This paper aims to lay theoretical foundations for such choices by formalizing preferences over runtime distributions. It might seem that we should simply prefer the algorithm that minimizes expected runtime. However, such preferences would be driven by exactly how slow our algorithm is on bad inputs, whereas in practice we are typically willing to cut off occasional, sufficiently long runs before they finish. We propose a principled alternative, taking a utility-theoretic approach to characterize the scoring functions that describe preferences over algorithms. These functions depend on the way our value for solving our problem decreases with time and on the distribution from which captimes are drawn. We describe examples of realistic utility functions and show how to leverage a maximum-entropy approach for modeling underspecified captime distributions. Finally, we show how to efficiently estimate an algorithm's expected utility from runtime samples.
\end{abstract}

\section{Introduction}\label{sec:intro}

Imagine that we need to solve a series of instances of a given computational problem such as SAT, MIP, TSP, or sorting, and we have a set of different algorithms to choose from. Suppose further that all these algorithms are guaranteed to return the correct solution but vary in the runtime they will take. Which algorithm should we choose? The obvious answer would be to choose the algorithm that returns the solution most quickly on average. But average runtime can be dominated by extremely rare, extremely long runtimes, making the task of choosing between algorithms rely on certifying the nonexistence of rare but very costly tail events. Does this describe the way we actually make choices between algorithms? Probably not. Consider the following example.

\begin{example}[Motivation]\label{ex:motivating}
    Suppose that we have 100 integer programs to solve and two algorithms to chose from. Algorithm $A$ solves the first 99 problems in 1 second, but runs the 100\textsuperscript{th} problem for 10 days without solving it. Algorithm $B$ runs all 100 problems for 10 days each without solving any of them.
\end{example}

In this case, we imagine that most readers would strongly prefer Algorithm $A$ to Algorithm $B$. However, the algorithms' average runtimes are unconstrained by the information given; e.g., $B$ could solve every problem in a bit more than 10 days, whereas $A$ could take 100 years to solve the last problem. Furthermore, imagine that both $A$ and $B$ contain a bug that prevents their long runs from terminating at all. This would cause both averages to become infinite, but we doubt that it would make most readers indifferent between the algorithms. Such preferences are not described by an average runtime scoring function. 

We can learn a second lesson from this example: we are at least sometimes able to express preferences between algorithms even when some runs are ``capped'' (i.e., right censored). Intuitively, the reason is that while the runtime of an algorithm may be large or unbounded, the impact it has on us is not: we will eventually terminate any sufficiently long run and rely instead on some backup plan. Thus, one algorithm run that would take a thousand years and another that would take a trillion years are functionally equivalent: neither stands any chance of running until completion. A delivery service must route vehicles, but cannot wait so long that solving the optimization problem delays its deliveries; at some point, it must send the drivers out to do the best they can. A chip manufacturer using a verification procedure to test a new CPU design can wait much longer for a solution---and probably never faces one moment at which waiting a bit longer would not be acceptable---but as time passes, managers or stakeholders will demand results and the risk of being scooped by a competitor will increase. At some point the manufacturer must decide either to ship their product unverified or to scrap the project entirely. 

Practitioners really do face such choices between algorithms. For the task of automated algorithm design, the field has seen a trend in past decades away from hand-crafting heuristic algorithms. Instead, the selection of algorithms optimized for practical performance is increasingly being approached as a machine learning problem: given an instance distribution, highly parameterized algorithms are configured to maximize empirical performance, just as a classifier's parameters are chosen to minimize empirical risk. Prominent examples of this approach include heuristic approaches for building algorithm portfolios \citep{Ric76,HubEtAl97,gomes01,Horvitz01,boosting-IJCAI,SATzilla-Full}, performing algorithm configuration, \citep{birattari2002racing,hutter2009paramils,hutter2011sequential,ansotegui2009gender,lopez2016irace} and beyond \citep{LagLit01,GagSch06,xu-aaai10a,kadioglu-ecai10,seipp-aaai15a}. There is also a growing literature on theoretically-grounded methods that offer performance guarantees \citep{gupta2017pac,kleinberg2017efficiency,kleinberg2019procrastinating,balcan2017learning,balcan2021much,WeGySz18,weisz2019capsandruns,Weisz0LGLSL20}. This shift to learning new algorithms rather than designing them by hand makes the choice of loss function crucial: as with any optimization problem, varying the objective function profoundly impacts which solution is returned. 

Contests for evaluating the practical performance of solvers for NP-hard problems are increasingly found as part of major AI conferences. Naturally, these require a scoring metric by which algorithms can be compared. Organizers of these competitions must specify scoring functions by which to evaluate algorithms, and are clearly aware that different metrics will lead to different rankings. It is trivial to define metrics that change the way algorithms are ranked. The difficulty is to define metrics that properly reflect what the organizers think of as ``good'' algorithms. The approach in the International SAT Competition\footnote{\url{http://www.satcompetition.org/}.} has essentially been to choose a metric and then retrospectively evaluate the rankings it leads to. The organizers have shown a clear awareness of the way different scoring metrics can affect rankings. Consider the SAT Competition of 2009.\footnote{\url{http://www.satcompetition.org/2009/spec2009.html}.} We quote part of their rationale for considering different metrics: \emph{One interesting question is if the speed of a faster solver can compensate for its failure to solve an instance. For example, assume solver A can solve 100 instances in 1000s (cumulated time) and solver B can solve the same 100 instances in only 100s. At this point, solver B is clearly better than solver A. Now assume solver A can solve one more instance than B. Which solver is the best? The answer is probably not unique and certainly depends on the user's applications and expectations.}

Some competitions have given up on finding the right evaluation metric altogether, and resorted to judging algorithm performance by ``a jury consisting of researchers with experience in computational optimization'',\footnote{\url{https://www.mixedinteger.org/2022}.} which highlights rather keenly the need for more principled foundations. Instead of choosing a scoring function and then asking in a vague way if it satisfies our idea of what it means to be a ``good'' algorithm, our paper provides a framework within which decision-makers such as the organizers of competitions can define scoring metrics \emph{a priori} that capture the properties of algorithms they actually care about.

Various other scoring functions have been widely used in the empirical algorithmics literature. One common choice is to score algorithms according to their \emph{capped} average runtime, as in the literature on black-box algorithm configuration approaches offering theoretical guarantees \citep{kleinberg2017efficiency,kleinberg2019procrastinating,WeGySz18,weisz2019capsandruns,Weisz0LGLSL20}. However, this treats capped runs as being virtually the same as runs that completed just before the captime. To address this, we can count runs that reach the captime $\kappa$ as having taken $c\cdot\kappa$ seconds, yielding a Penalized Average Runtime (PAR) \citep{hutter2009paramils} (where, e.g., $c=2$, as in the 2021 SAT Competition \citep{froleyks2021sat}). However, $c$ is an arbitrary parameter to choose, and when $c>1$ an algorithm's penalized runtime can actually \emph{fall} as the captime increases, since fewer runs will cap. 

Other scoring functions that balance the risk of timeouts with the desirability of short runs have been explored using the tools of survival analysis \cite{tornede2020run2survive}. For instance, minimizing the expectation of runtime raised to a large exponent will penalize long runs more heavily than short runs, thus favouring algorithms that are less at risk of timing out. A goal of our work is to formalize the reasoning behind using such functions.

This paper seeks to formalize preferences over runtime distributions using an axiomatic approach. In \cref{sec:axioms} we present six constraints on preferences over runtime distributions and prove that these axioms imply a general rule which describes our preferences in general. Our technical approach draws on the expected utility construction of \citet{vonneumann1944theory} with two key modifications. First, we add two additional runtime-specific axioms, amounting to the assertion that we actually care about solving our problem and that we will tend to want to solve it more quickly. Second, the fact that runs can be censored at arbitrary points requires some nontrivial changes to the classical axioms. In \cref{sec:using} we work through instantiations of our utility functions for different practical applications and show how to choose captime distributions in settings where these are not known exactly by maximizing entropy. In \cref{sec:sampling}, we show that the time required to $\epsilon$-estimate the score of an algorithm from capped samples depends on $\epsilon$ and on the inverse of our utility function, but not on the algorithm's average runtime. Finally, in \cref{sec:results} we present some real-world examples where the choice of utility function really is important and changes our conclusions about which algorithm is considered ``best.''

%%%%%%%%%%%%%%%%%%%%%%%%%%%%%%%%%%%%%%%%%%%%%%%%%%%%%%%%%%%%%%%%%%%%%%%%%

\section{Preferences Over Runtime Distributions}\label{sec:axioms}

Let $\mathcal{I}$ be a probability distribution over instances of some computational problem. Let $\A$ be a set of (potentially randomized) algorithms, where each $A \in \A$ either runs forever or produces a solution of identical quality\footnote{We consider algorithms that return different quality solutions in \cref{app:solutionquality}, but this extension sheds little additional light on the problem since it must simply appeal to the existence of a runtime/quality tradeoff function. The binary setting we describe here is in some sense fundamental; it arises often in practice (e.g., any decision problem), and our formalism and results are easily extended to the general solution-quality case.} when given any input sampled from $\mathcal{I}$. For a given algorithm $A$ we will use $t$ to denote the algorithm's running time when presented with an instance from $\mathcal{I}$; $t$ is thus a random value that depends on the instance sampled from $\mathcal{I}$ and on any other source of randomness (e.g., on the choice of random seed, on allocation choices made by the operating system, etc.). Some algorithms may fail to terminate on some inputs or with some random seeds; in such scenarios $t$ is infinite. Each algorithm is thus associated with a \emph{runtime distribution}, a probability distribution over the positive extended real numbers $[0, \infty]$. We overload notation and identify each algorithm $A$ with its runtime distribution,\footnote{We will have no further need of $\mathcal{I}$ beyond its role in defining runtime distributions.} because this is the only fact about each algorithm that concerns us.
It is useful to consider algorithms which always take a fixed, deterministic amount of time regardless of their input or random seed. The runtimes of such algorithms are Dirac delta distributions; we use $\delta_x$ to denote the distribution that returns $x$ with probability 1. 

Let $K$ be a probability distribution over the amount of time we will have to run our algorithm; we denote a captime sampled from $K$ as $\kappa$. $K$ can be deterministic: $K = \delta_{\kappa}$. For any distribution $X$, we use the the function $F_X$ to denote the cumulative distribution function (CDF) of $X$. 

Our goal is to define a scoring function that represents our preferences for elements from $\A$. This will turn out to correspond with the expectation of a utility function. 
\citet{vonneumann1944theory} showed how to derive a real-valued scoring function over arbitrary discrete outcomes from a set of more basic assumptions about the properties of a preference relation. We employ many of the same building blocks to derive a scoring function appropriate for our setting. Our first four axioms correspond closely to theirs, although \cref{axiom:independence} in particular is subtly different. Beyond this, we introduce two additional axioms that are natural in the runtime setting.

Given captime distribution $K$, let $\prefgteq{K}$ be a binary relation over pairs of elements of $\A$ that describes our preferences among algorithms (runtime distributions) when faced with captime distribution $K$. For $A,B \in \A$, we will use $A \prefgteq{K} B$ to denote the proposition that we weakly prefer algorithm $A$ to algorithm $B$, given captime distribution $K$. Similarly, $A \prefgt{K} B$ denotes the proposition that we strictly prefer algorithm $A$ to $B$ given $K$, and $A \prefeq{K} B$ denotes the proposition that we are indifferent between the two.\footnote{The relations $\prefgt{K}$ and $\prefeq{K}$ are used only for notational convenience and are derived from $\prefgteq{K}$: we have $A \prefgt{K} B$ iff $A \prefgteq{K} B$ and not $B \prefgteq{K} A$; similarly, $A \prefeq{K} B$ iff $A \prefgteq{K} B$ and $B \prefgteq{K} A$.} What properties should we insist that our relation $\prefgteq{K}$ must have? Our first axiom asserts that our relation is acyclic. 

\begin{axiom}[Transitivity]\label{axiom:transitivity} If $A \prefgteq{K} B$ and $B \prefgteq{K} C$, then $A \prefgteq{K} C$.
\end{axiom}

To explain our second axiom, we must introduce notation to describe how algorithms (distributions) can be combined to form new algorithms (compound distributions), for example by creating a new algorithm that uses coin flips to decide which of a set of existing algorithms to run. We define two operations that describe the way some of these combinations take place. For $0 \geq p \geq 1$ the \emph{mixing} operation creates convex combinations of runtime distributions. We use the notation $[p : A, (1-p) : B]$ to denote the new runtime distribution induced by drawing a runtime from $A$ with probability $p$ and from $B$ with probability $1 - p$. The next two axioms describe our preferences over such mixtures. Monotonicity says that if we prefer distribution $A$ to distribution $B$, then we prefer mixtures over $A$ and $B$ that give more weight to $A$ than $B$. That is, we prefer mixtures that give us ``more of a good thing.''

\begin{axiom}[Monotonicity]\label{axiom:monotonicity} If $A \prefgteq{K} B$ then for any $p,q \in [0,1]$ we have ${[p:A , (1-p):B]} \prefgteq{K} {[q:A , (1 - q):B]}$ if and only if $p \ge q$.
\end{axiom}

Continuity says that whenever we have preferences over three runtime distributions, there exists a mixture between the most- and least-preferred distributions that makes us indifferent between that mixture and the middle distribution.

\begin{axiom}[Continuity]\label{axiom:continuity} If $A \prefgteq{K} B \prefgteq{K} C$, then there exists $p \in [0,1]$ such that $B \prefeq{K} {[p:A , (1-p):C]}$.
\end{axiom}

Given any mapping $M$ that associates a distribution $M(t, \kappa) \in \A$ with the runtime--captime pair $t, \kappa$, the \emph{compounding} operation constructs a new distribution $[M(t, \kappa) \st t \sim A, \kappa \sim K] \in \A$, which first draws $t$ from $A$ and $\kappa$ from $K$, then returns a runtime drawn from the corresponding $M(t, \kappa)$. (We might think of this compound distribution as an algorithm that first samples values for $t$ and $\kappa$ from $A$ and $K$, then runs a master algorithm with some corresponding parameter configuration, giving $M(t, \kappa)$.) Our next axiom describes the way our preferences for sure outcomes and captimes affect our preferences for general algorithms and captime distributions by relating them through compound distributions. This axiom is commonly called Independence because of the way it assures us that our preferences for distributions can be determined from the parts of those distributions, \emph{independently} of one another; no confounding factors are created when we nest distributions. 

\begin{axiom}[Independence]\label{axiom:independence}
    If $\delta_t \prefeq{\delta_{\kappa}} M(t, \kappa)$ for all $t, \kappa$, then $A \prefeq{K} [M(t, \kappa) \st t \sim A, \kappa \sim K]$.\footnote{We state Independence this way for clarity of notation, but strictly it only needs to hold when $M(t, \kappa) = \big[ p: \delta_0, (1- p): \delta_{\infty} \big]$ for some $p$ (that may depend on $t$ and $\kappa$).}
\end{axiom}

Observe that if $M(t, \kappa) = \delta_t$ then it is trivially obvious why Independence should hold because the compound distribution $[\delta_t \st t \sim A, \kappa \sim K]$ is exactly equal to $A$. When $M(t, \kappa) \prefeq{\delta_{\kappa}} \delta_t$ but $M(t, \kappa) \not= \delta_t$, Independence says that if $t$ is drawn from $A$ and $\kappa$ is drawn from $K$, then we should be indifferent (under captime distribution $K$) between a runtime of $t$ and a runtime drawn from $M(t, \kappa)$, since this is precisely what it means to be indifferent (under a sure captime of $\kappa$) between $\delta_t$ and $M(t, \kappa)$. In other words, if we are indifferent between each of a set of outcome pairs,
we are also indifferent between mixtures that equally weight respective elements of these pairs. 

Our first three axioms can be found in the classical von Neumann-Morgenstern setup. Independence has an analogue in the classical setup but needs some adjustments to address the facts that (a) base outcomes already correspond to distributions from which we can sample; (b) our preference relation is defined with respect to a captime distribution; and (c) the axiom relates preferences under deterministic captime distributions $\delta_\kappa$ to those under a general distribution $K$. 
Please see \cref{app:independence} for an extended discussion of the history of this axiom and our own variant's relationship to this history, as well as a survey of the literature on why preferences might violate this axiom and an argument that the runtime setting is different (because it introduces the ability to limit losses via capping). Together, these four axioms already suffice to show the existence of a utility function whose expectation captures our preference relation. 

We now introduce two further axioms, which capture additional properties inherent in our runtime distribution setting and hence constrain the utility function's form. First, Eagerness says that a deterministic algorithm is preferable to any algorithm that always takes at least as long. 

\begin{axiom}[Eagerness]\label{axiom:eagerness} For any $t \le t'$, if the support of $A$ is contained in $[t, t']$, then $\delta_t \prefgteq{K} A \prefgteq{K} \delta_{t'}$.
\end{axiom}

Second, the Relevance axiom states that we strictly prefer deterministically solving our problem to deterministically failing to solve it.

\begin{axiom}[Relevance]\label{axiom:relevance} $\delta_t \prefgt{\delta_{\kappa}} \delta_{t'}$ for all $t < \kappa \le t'$.
\end{axiom}

These axioms imply that our preferences for algorithms correspond to a scoring function. Before stating our main theorem, we define an important function $p$ that describes our propensity for risk.

\begin{definition}\label{def:p}
    For any $t$ and $\kappa$, the function $p : \R^+ \times \R^+ \to [0,1]$ is defined as follows. Set $p(t, \kappa) = 0$ if $t \ge \kappa$, and otherwise set $p(t, \kappa)$ to be the value that satisfies $\delta_{t} \;\prefeq{\delta_{\kappa}}\; \big[ p(t, \kappa) :  \delta_0 \,,\, \big(1 - p(t, \kappa) \big) : \delta_{\infty} \big]$.
\end{definition}
Since Eagerness tells us that $\delta_{0} \prefgteq{\delta_{\kappa}} \delta_{t} \prefgteq{\delta_{\kappa}} \delta_{\infty}$ when $t < \kappa$, Continuity ensures that $p(t, \kappa)$ exists for all $t,\kappa$, and Monotonicity ensures that it is unique. So, $p$ is well-defined.

We can now state that $p$ is essentially (i.e., up to affine transforms) the only utility function that corresponds with our preferences, and can infer certain properties of its shape.

\begin{theorem}\label{thm:main}
    If our preferences follow \cref{axiom:relevance,axiom:transitivity,axiom:continuity,axiom:monotonicity,axiom:independence,axiom:eagerness}, then a function $u$ satisfies 
    \begin{small}
    \begin{align}\label{eqn:thm:main:uutil}    
        A \;\prefgteq{K}\; B  \iff \Emath_{t \sim A, \kappa \sim K}\big[u(t , \kappa)\big] \ge \Emath_{t \sim B, \kappa \sim K}\big[u(t , \kappa)\big]
    \end{align}
    \end{small}for any runtime distributions $A$ and $B$ and any timeout distribution $K$ if and only if there are constants $c_0$ and $c_1 > 0$ such that $u(t, \kappa) = c_1 p(t, \kappa) + c_0$. Furthermore, $p(0, \kappa) = 1$ (maximum achieved at $t=0$), $p(t, \kappa) \ge p(t', \kappa)$ for all $t \le t'$ (monotonically decreasing), $p(t, \kappa) > 0$ for all $t < \kappa$ (strictly positive), $p(\kappa, \kappa) = 0$ (minimum achieved at $t=\kappa$). 
\end{theorem}

Please see \cref{app:proofs} for the proof. The first three von Neumann-Morgenstern axioms (\cref{axiom:transitivity,axiom:monotonicity,axiom:continuity}) imply the existence of the function $u$. The fourth, \cref{axiom:independence}, implies that the utility of an algorithm is the expectation of $u$ over the algorithm's runtime distribution. The final two, novel axioms (\cref{axiom:eagerness,axiom:relevance}) give the function $u$ its particular form. 

The function $p$ is our ``fundamental'' utility function. Any other valid utility function must be a positive linear transformation of $p$. (It can be seen from linearity of expectation that the constants $c_1$ and $c_0$ in no way affect the ordering of our preferences.) The value $p(t, \kappa)$ reflects our feelings about an algorithm run that takes $t$ seconds when $\kappa$ seconds were available, and corresponds to a measure of how happy we are when faced with a gamble that either gives us our solution immediately or requires us to spend $\kappa$ seconds to learn nothing. We can consider possible forms of $p$ by reasoning about our feelings regarding (i) spending $t$ seconds to solve the problem when $\kappa$ seconds were available, (ii) wasting $\kappa$ seconds to accomplish nothing, and (iii) how important it is to solve our problem. Having thus fixed a specific form for $p$, our algorithm scoring function will be the expected value of $p$ with respect to runtime distribution $A$ and captime distribution $K$. 

Since this paper is ultimately about formalizing our preferences between \textit{algorithms}, it is natural to consider the case where we can learn about runtime distributions only through sampling (discussed in \cref{sec:sampling}), but where we do know something about the captime $\kappa$---its mean, say, or the fixed and bounded range in which it falls. The next section discusses how we might give specific form to the function $p$ in practice.

%%%%%%%%%%%%%%%%%%%%%%%%%%%%%%%%%%%%%%%%%%%%%%%%%%%%%%%%%%%%%%%%%%%%%%%%%

\section{Instantiating Utility Functions in Practice}\label{sec:using}

We now present a series of examples with the aim of guiding anyone who wishes to apply our framework to their own particular setting. We present these in the context of two key questions we should ask ourselves: (1) Are we sure about what captime we will face?; and (2) Does an immediate solution give us the same utility as a solution later in the future?

\subsection{The Case of Known Captime Distributions}

We begin with the simplest scenario, where the answer to both of the above questions is affirmative. 

\begin{example}[Known captime, step-function utility]\label{ex:known-constant}
    Suppose we know we face a fixed captime $\kappa_0$ (i.e., $K = \delta_{\kappa_0}$) and we will receive the same value as long as we solve our problem before the captime. If we set $c_1 = 1$ and $c_0 = 0$, we have that $u(t, \kappa) = 1$ for $t < \kappa$ and $0$ otherwise. So our score for an algorithm $A$ is $F_A(\kappa_0)$, the value of $A$'s CDF at $\kappa_0$. In other words, the best algorithm for this utility function is simply the one that is most likely to finish. Considering again the motivating scenario of \cref{ex:motivating}, algorithm $A$ is at least as good as algorithm $B$ for any $\kappa_0 < 10$ days (if $\kappa_0 \ge 10$ days, the problem is indeterminate). 
\end{example}

The scoring function in \cref{ex:known-constant} is simple and intuitive, and similar metrics are commonly used in practice (under names like ``number of instances solved''). Choosing to optimize this binary utility function implies that we are indifferent between learning an answer to our problem immediately and learning it $t$ seconds (or minutes or hours) from now, so long as $t$ is less than $\kappa_0$. Next, if we have to pay for runtime and get paid for solving our problem, then we might instead be interested in how much money we can make from an algorithm.

\begin{example}[Known captime, linear utility for compute time]\label{ex:known-lincompute}
	Suppose that we face a fixed and known captime $\kappa_0$, that we pay $\alpha$ dollars for each second of compute, and that we earn $v$ dollars if we are able to solve our problem. Further suppose that we have a linear value for money and that money is the only variable we care about. If we complete a run in $t < \kappa_0$ seconds, we earn $v - \alpha \cdot t$ dollars. If the run caps, we lose $-\alpha \cdot \kappa_0$ dollars. So our utility function is $u(t,\kappa_0) = v - \alpha \cdot t$ if $t < \kappa_0$ and $u(t, \kappa_0) = -\alpha \cdot \kappa_0$ otherwise. Since $u(t,\kappa) = c_1 p(t,\kappa) + c_0$, by setting $c_1 = v + \alpha \cdot \kappa_0$ and $c_0 = -\alpha \cdot \kappa_0$ we can normalize to the interval $[0,1]$. We find that $p(t,\kappa_0) = \frac{v + \alpha \cdot (\kappa_0 - t)}{v + \alpha \cdot \kappa_0}$ if $t < \kappa_0$, and $p(t, \kappa_0)$ = 0 otherwise. 
\end{example}

We can see the three aspects of the function $p$ (mentioned at the end of \cref{sec:axioms}) explicitly: (i) spending $t$ seconds to solve the problem costs us $-\alpha \cdot t$, (ii) wasting $\kappa_0$ seconds costs $-\alpha \cdot \kappa_0$, and (iii) the importance of solving the problem is given by $v$, the value we gain from solving it. To see the full strength of our method we can compare the above scoring function to the simple capped average runtime, which is commonly used in practice and may seem like an obvious choice of objective, but which actually implies a logical contradiction. If runtime is what we care about, then whether each run caps or not, our utility is proportional to the time of the run (perhaps also penalizing runs that cap, e.g., PAR10). This gives the utility function $u(t, \kappa_0) = -t$ if $t < \kappa_0$, and $u(t,\kappa_0) = - \kappa_0$ otherwise. Choosing to optimize this capped average has a strange implication that can be seen when we compare this utility function to the one above. Scaling so that the parameter $\alpha = 1$, the only difference between these two utility functions is the parameter $v$, the value we gain from solving our problem. For the capped average objective, this term is 0. This amounts to the assertion that we gain no value from solving our problem! Clearly this is ridiculous (otherwise, why solve it?). We may also be interested in considering more general cost and value functions, as in the next example.

\begin{example}[Known captime, linear utility for money]\label{ex:known-linmoney}
    Suppose that we face a fixed and known captime $\kappa_0$, that we pay $\cost(t)$ dollars for $t$ seconds of compute, and that if we are able to solve our problem within $t$ seconds, we can sell the answer and earn $\revenue(t)$ dollars. So if we solve our problem in $t < \kappa_0$ seconds, we will earn a profit of $\revenue(t) - \cost(t)$ dollars; if not, we will lose $- \cost(\kappa_0)$ dollars. If we only care about money, and we care about money in a linear way, then our utility is simply proportional to the number of dollars we earn. Setting $c_1 = \revenue(0) + \cost(\kappa_0) - \cost(0)$ and $c_0 = - \cost(\kappa_0)$, we find that $p(t,\kappa_0) = \frac{\revenue(t) + \cost(\kappa_0) - \cost(t)}{\revenue(0) + \cost(\kappa_0) - \cost(0)}$ if $t < \kappa_0$, and $p(t, \kappa_0) = 0$ otherwise.
\end{example}

The preference for money in \cref{ex:known-linmoney} implies a specific utility function. The denominator $\revenue(0) + \cost(\kappa_0) - \cost(0)$ can be interpreted as money ``on the table''; $\revenue(0)$ is the maximum we stand to earn and $\cost(\kappa_0) - \cost(0)$ is the maximum cost we can save. The numerator $\revenue(t) + \cost(\kappa_0) - \cost(t)$ represents the portion of this total available money that we were actually able to collect by finishing our run at time $t$. So $p(t, \kappa_0)$ is the proportion of the total available money that we were able to earn. We can again see the three aspects of our risk preferences mentioned above: (i) the value of spending $t$ of $\kappa_0$ seconds to solve our problem is $\revenue(t) - \cost(t)$; (ii) the loss from wasting $\kappa_0$ seconds to accomplish nothing is $\cost(\kappa_0)$; and (iii) the value of the information to be gained from solving our problem is $\revenue(0) - \cost(0)$. We can generalize the above to arbitrary functions.  

\begin{example}[Known timeout, general utility]\label{ex:known-general}
    Suppose $V(t, \kappa_0)$ is some arbitrary benefit we assign to spending $t$ out of the $\kappa_0$ available seconds to solve our problem, $W(\kappa_0)$ is some arbitrary loss we assign to wasting $\kappa_0$ seconds to accomplish nothing, and $V(0, \kappa)$ is the benefit we gain from solving our problem. Setting $c_1 = V(0, \kappa_0) + W(\kappa_0)$ and $c_0 = -W(\kappa_0)$, we find that $p(t, \kappa_0) = \frac{V(t, \kappa_0) + W(\kappa_0)}{V(0, \kappa_0) + W(\kappa_0)}$ for $t < \kappa_0$, and $p(t, \kappa_0) = 0$ otherwise. 
\end{example}

We now turn to the case where we are uncertain about captime (i.e., where $K$ is not deterministic) in the special case of step-function utilities described in \cref{ex:known-constant}. We do this both because this is the simplest case and because we often really do not care about how long an algorithm run takes given that it completes before our captime (e.g., we do not pay for compute or electricity; our solution will not be used until some point in the future anyway). Indeed, we will see that captime distributions can induce utility functions whose expected values decrease smoothly with time even in this case, and so captime distributions can be seen as one principled way of modeling the way utility depends on time.

\begin{example}[Unknown captime, step-function utility]\label{ex:unknown-constant}
    Suppose we have a step-function utility as in \cref{ex:known-constant} but face captime distribution $K$. Setting $c_1= 1$ and $c_0 = 0$ again gives $u(t, \kappa) = 1$ for $t < \kappa$ and 0 otherwise, but this time $u$ is not fixed with respect to $\kappa$. Our scoring function for an algorithm $A$ is the expectation over $A$'s runtime of the probability that $A$ will finish before it times out, i.e., $K$'s survival function. Thus, algorithm $A$'s score is given by $\Emath_{t \sim A} \big[ 1 - F_K(t) \big]$.
\end{example}

Indeed, the function $F_K$ plays an integral part in utility functions beyond the step-function case as well. For example, suppose we model $\kappa$ as being distributed uniformly in the range from $0$ to $\kappa_0$ seconds. Then the distribution $K$ has CDF $F_{unif}(t) = t / \kappa_0$ for $t < \kappa_0$ and $F_{unif}(t) = 1$ otherwise. The utility function whose expectation we should maximize in this case is linear: $u_{unif}(t) = 1 - \frac{t}{\kappa_0}$ if $t < \kappa_0$, and $u_{unif}(t) = 0$ otherwise. 

\subsection{The Case of Underspecified Captime Distributions}\label{sec:maxent}

Often, we may know some properties of $K$ without knowing it exactly. In such cases, we advocate employing the method of maximum entropy of \citet{jaynes1957information}, which allows us to incorporate such knowledge without introducing additional, unjustified assumptions. Entropy is a measure of the average information content that would be revealed to us by observing some unknown quantity. The principle of maximum entropy tells us to use the distribution for which this value is highest, subject to satisfying whatever restrictions we have. It allows us to impose only the assumptions we want to make when assigning probability mass to our captime distribution. In this sense, the principle of maximum entropy is an application of Occam's razor. We can consider how different pieces of information imply different distributions $K$ that affect the form of our algorithm scoring function through their CDF $F_K(t)$. 

The simplest restriction our prior information could place on the distribution $K$ is a bound on its support. If we need a solution to our problem before some fixed deadline, but we do not perfectly trust our equipment and realize that a failure may leave us with somewhat less time, then we know $\kappa$ falls in some fixed range, but nothing more. With no restrictions on the prior $K$ beyond having bounded support, the maximum entropy captime distribution is the same uniform distribution we derived above (please see \cref{app:maxent} for all maximum entropy derivations): $u_{\text{unif}}(t) = 1 - \frac{t}{\kappa_0}$ if $t < \kappa_0$ and $u_{\text{unif}}(t) = 0$ otherwise.

Perhaps we do not know an upper bound on the timeout we will face, but instead know that we will have $\kappa_0$ hours on average (based on the average lifetime of our equipment, say). In this case the maximum entropy distribution for all priors with the condition that $\E_{\kappa \sim K}[\kappa] = \kappa_0$ is an exponential distribution with rate parameter $1/\kappa_0$, so we should optimize an exponential utility function: $u_{\text{exp}}(t) = e^{-t/\kappa_0}$.

\begin{figure*}[t]
    \centering
    \includegraphics[width=.93\textwidth]{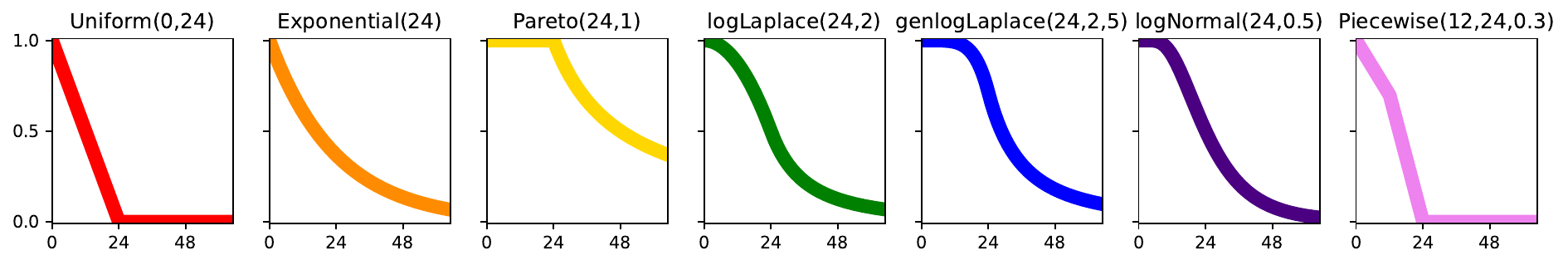}
    \vspace{-1em}
    \caption{\small{Utility functions from \cref{sec:maxent}, obtained from different maximum-entropy prior distributions.}}
    \label{fig:utilgrid}
\end{figure*}

Maybe we do not quite know the average time limit, but we do know its expected order of magnitude. This amounts to a constraint on the expectation of the log of the timeout: $\E_{\kappa \sim K}[\log{(\kappa / \kappa_0)}] = 1 / \alpha$ where $\kappa > \kappa_0$ (i.e., we know the order of magnitude measured in units of $\kappa_0$). The maximum entropy distribution under this constraint is a Pareto with shape parameter $\alpha$, so we should optimize a geometric utility function: $u_{\text{Pareto}}(t) = 1$ if $t < \kappa_0$ and $u_{\text{Pareto}}(t) = \big( \frac{\kappa_0}{t} \big)^{\alpha}$ otherwise.

We summarize how a decision maker might identify an appropriate utility function by reasoning about both their knowledge about the captime and any costs they incur for the passage of time in \cref{fig:summary_table}.

\input{summary_table}

\begin{example}[Pareto timeout, constant utility]\label{ex:pareto-constant}
    Suppose that we face no decrease in utility from time spent running our algorithm, but that we do not know the distribution of $\kappa$ with certainty, only that we expect it to be on the order of seconds. If we are also measuring time in seconds, we can interpret this condition as the restriction $\E_{\kappa \sim K}[ \log(\kappa) ] = 1$, which implies a Pareto distribution for $K$ with parameters 1 and 1. Setting the constants $c_1=1$ and $c_0=0$ we get the score function $\Emath_{t \sim A, \kappa \sim K}[u(t, \kappa)] = F_A(1) + \Emath_{t \sim A} \big[ \frac{1}{t} \st t \ge 1 \big] \big( 1 - F_A(1) \big)$. Further, if it happens to be the case (as in \cref{ex:motivating}) that $A$ always takes at least 1 second, so that $F_A(1) = 0$, our scoring function becomes $A$'s expected inverse runtime $\Emath_{t \sim A} [ 1 / t]$. Applied to the two algorithms in \cref{ex:motivating}, this gives $\Emath_{t \sim A}[1 / t] \ge 0.99$, while $\Emath_{t \sim B}[1 / t] \le 1/864000$.
\end{example}

The geometric (Pareto) utility function offers an appealing, alternative interpretation. Consider how our utility would change if our runtime doubled from $t$ to $2t$. Fixing $\alpha = 1$ for simplicity, we see that $u(2t) = u(t) / 2$: doubling runtime halves our utility. Different values of $\alpha$ would give different rates of geometric progression. The parameter $\kappa_0$ serves to calibrate our runtimes, determining our units of measurement.\footnote{Mathematically, $\kappa_0$ serves to remove the units from inside the logarithm, which is a transcendental function.} Are we expecting a runtime of seconds, hours or days? If hours, the Pareto utility pays no consideration to runtimes smaller than one hour; these are indistinguishable to us and all represent perfect utility. But we do not need to be so indiscriminate of small runtimes. The Pareto utility was derived from an order-of-magnitude condition like $\E_{\kappa \sim K}[\log(\kappa / \kappa_0) \st \kappa \ge \kappa_0] = 1 / \alpha$. An equivalent left-tail condition would be $\E_{\kappa \sim K}[\log(\kappa_0 / \kappa) \st \kappa < \kappa_0] = 1 / \alpha$. If we insist on continuity at $\kappa_0$ (implying a smooth utility function), the maximum entropy distribution under these conditions is a log-Laplace\footnote{$X$ follows a log-Laplace distribution if $\log X$ follows a Laplace distribution, analogous to log-normal.} with parameters $\log \kappa_0$ and $\alpha$, giving a utility function with geometric decay above $\kappa_0$ as well as geometric growth (towards unity) below it: $u_{LL}(t) = 1 - \frac{1}{2} \big( \frac{t}{\kappa_0} \big)^{\alpha}$ if $t < \kappa_0$ and $u_{LL}(t) = \frac{1}{2}\big( \frac{\kappa_0}{t} \big)^{\alpha}$ otherwise. 

The log-Laplace distribution divides its probability mass equally between values greater and less than $\kappa_0$ (i.e., it assumes the probability that we get an extension is equal to the probability that we face a mishap). This might not be the case (e.g., if our client is a stickler and our servers are old). We can insist that $\Pr_{\kappa \sim K}(\kappa < \kappa_0) = p$, and impose similar but more flexible order-of-magnitude tail conditions: $\E_{\kappa_\sim K}[ \log(\kappa_0/\kappa) \st \kappa < \kappa_0] = 1/\beta$ and ${\E_{\kappa_\sim K}[ \log(\kappa/\kappa_0) \st \kappa \ge \kappa_0]} = 1/\alpha$. The maximum entropy distribution with these conditions gives us control over how much probability mass we place on either side of $\kappa_0$, and over the decay rates for timeouts that deviate from $\kappa_0$. However, if we want a distribution that is continuous at $\kappa_0$ (so that $u$ is smooth) we find that we require $p = \frac{\alpha}{\alpha + \beta}$. The resulting distribution is a \emph{generalized} log-Laplace: $u_{GLL}(t) = 1 - \frac{\alpha}{\alpha + \beta} \big( \frac{t}{\kappa_0} \big)^{\beta}$ if $t < \kappa_0$ and $u_{GLL}(t) = \frac{\beta}{\alpha + \beta}\big( \frac{\kappa_0}{t} \big)^{\alpha}$ otherwise. 

The Pareto and log-Laplace distributions constrain the mean absolute deviation (in log space), but we could choose to constrain the more familiar squared deviation instead. Fixing $\E[( \log (\kappa / \kappa_0) )^2] = \sigma^2$, with the centering condition $\E[ \log (\kappa / \kappa_0) ] = 0$, gives a log-normal maximum entropy distribution and an error-function utility: $u_{LN}(t) = \frac{1}{2} - \frac{1}{2} \erf\big( \frac{\log (t/\kappa_0)}{\sqrt{2} \sigma} \big)$.

\begin{figure*}[t]
    \centering
    \includegraphics[width=.73\textwidth]{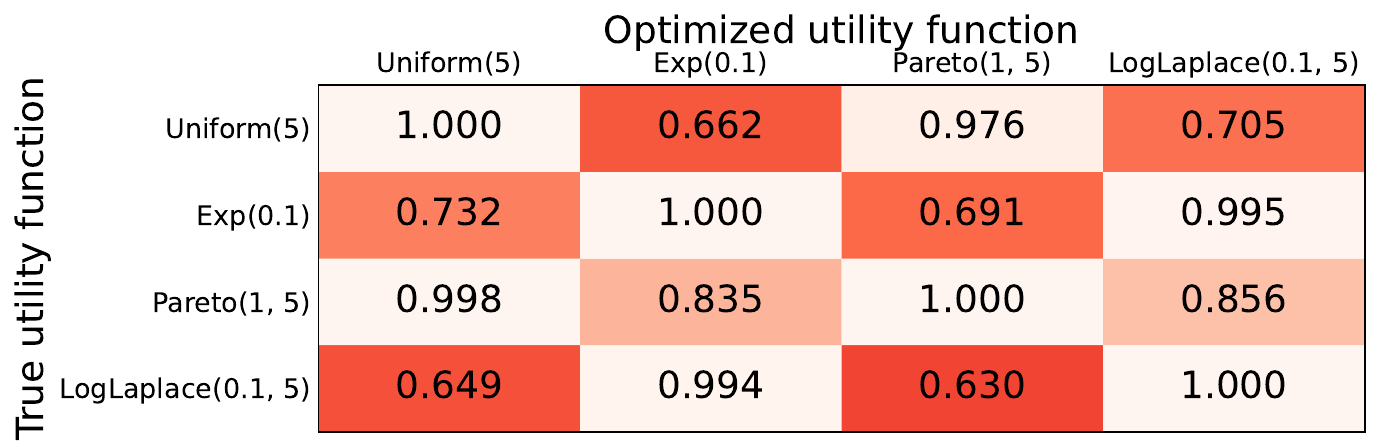}
    \vspace{-.15cm}
    \caption{\small{Failing to optimize the right function can yield significantly less utility. We optimized the \texttt{minisat} solver according to one utility function and then assessed its performance according to another. Rows indicate true utility functions; columns indicate the utility function used to perform the optimization; each function was normalized to have maximum 1. Colour intensity indicates degree of suboptimality. }}
    \vspace{-.15cm}
    \label{fig:cross_util}
\end{figure*}

Finally, we can incorporate tail constraints as maximum entropy conditions. Maybe we have $\kappa_0$ hours and we know our server might fail, but it is brand new and we know that the chance it will fail within the first $\kappa_0$ hours is very low, so that $\Pr_{\kappa \sim K}(\kappa < \kappa_1) \le \delta$. Then we should optimize the piecewise linear utility function: $u_{\text{piece}}(t) = 1 - \frac{\delta t}{\kappa_1}$ if $t < \kappa_1$, $u_{\text{piece}}(t) = \frac{(1 - \delta)(\kappa_0 - t)}{\kappa_0 - \kappa_1}$ if $\kappa_1 \le t < \kappa_0$, and $u_{\text{piece}}(t) = 0$ otherwise.

\cref{fig:utilgrid} illustrates the utility functions we have discussed in this section. Overall, the method of maximum entropy gives us a principled way to adopt partial knowledge about $K$ into our scoring function; of course, it applies similarly beyond the case of step-function utilities.

%%%%%%%%%%%%%%%%%%%%%%%%%%%%%%%%%%%%%%%%%%%%%%%%%%%%%%%%%%%%%%%%%%%%%%%%%
\section{Estimating Expected Utility from Samples}\label{sec:sampling}

So far we have talked about choosing between algorithms in the case where their runtime distributions are known. In practice, we must estimate runtime distributions via sampling and try to make high-probability claims about the distributions based on these estimates. In this setting, each sample we draw imposes a runtime cost equal to its (capped) value. Thus, if we want to be efficient in our estimation procedure, it is not the number of samples that concerns us, but the sum of their values. This raises the question of how cheaply we can estimate an algorithm's score.

Suppose $t_1, \ldots, t_m$ are runtimes sampled from $A$. We do not get to observe each $t_j$, but instead observe \emph{capped} runtimes $t_j(\kappa) = \min\{ t_j, \kappa \}$. We have a utility function $u(t) = u_{K}(t) = \E_{\kappa \sim K}[u(t, \kappa)]$ incorporating our knowledge of the captime distribution $K$, as in \cref{sec:maxent}. The utility values of the capped runtimes we observe are $u(t_j(\kappa))$, and we want to know if these are a good estimate of the true expected utility. The problem's saving grace is that we do not need to learn about the distribution of runtimes, we need to learn about the distribution of utilities, which is a monotonically decreasing function of runtime, bounded from below by 0. If we are estimating average runtime, long but exceedingly unlikely runs can always impact the expected value, and searching for them (to rule them out) becomes important. When estimating utilities on the other hand, long runs become less and less important since the utility they contribute approaches 0 (as does their likelihood of occurring) and so they contribute virtually nothing to the expectation. Thus, expected utility can be accurately determined from capped samples even when expected runtime cannot. For simplicity, we assume in this section that $u$ is invertible and bounded in $[0,1]$. \cref{thm:timebound} shows that the total time required to estimate an algorithm's score depends on $u^{-1}(\cdot)$, and on parameters $\epsilon$ and $\delta$, but has no dependence on $A$'s average runtime. That is, regardless of $A$'s runtime distribution---even if $A$'s mean runtime is infinite---we can accurately estimate its true expected utility from capped samples, with the number of samples required and the captime depending on $\epsilon$, $\delta$, and $u$.

\begin{theorem}\label{thm:timebound}
    The time to estimate the score of an arbitrary algorithm $A$ to within an $\epsilon$ additive factor with probability $1 - \delta$ will be greater than $m \cdot u^{-1}(2\epsilon)$ in the worst case, but always less than $m \cdot u^{-1}(\epsilon/2)$, where $m = \frac{\ln(2/\delta)}{2}\big(\frac{2-\epsilon}{\epsilon}\big)^2$.
\end{theorem}

See \cref{app:proofs} for a proof. Briefly, the captime is set large enough that runtimes beyond it are too small to matter. With this captime it is necessary and sufficient to do $m$ runs.

\section{The Effect of Different Utility Functions}\label{sec:results}

\begin{figure*}[t]
    \centering
    \includegraphics[width=.78\textwidth]{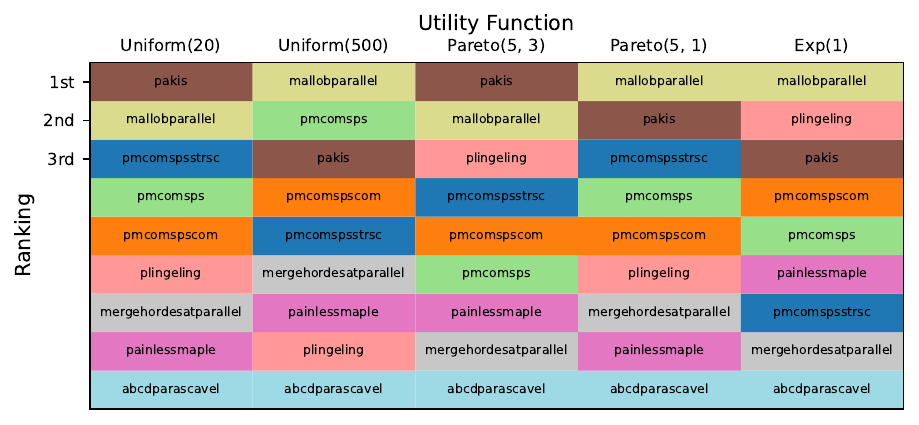}
    \vspace{-.15cm}
    \caption{\small{Results from the Parallel Track of the 2021 International SAT Competition under different utility functions.  Colours correspond to different solvers. Although various patterns emerge, each utility function yields a different top-three ranking.}}
    \vspace{-.15cm}
    \label{fig:satresults}
\end{figure*}

In \cref{sec:axioms} we showed that our preferences for algorithm runtime distributions will imply specific utility functions, which in turn imply specific algorithm scoring functions. But does this matter? Will our choices actually be affected? In this section we argue that  it does indeed matter. We demonstrate this in two settings: automated algorithm configuration and counterfactual analysis of the 2021 International SAT Competition\footnote{Code to reproduce all figures can be found at \url{https://github.com/drgrhm/formalizing-preferences}}.

\paragraph{Algorithm Configuration.} We considered a dataset due to \citet{WeGySz18} which evaluated 972 randomly-sampled configurations of the \texttt{minisat} \citep{sorensson2005minisat} SAT solver on 20118 instances generated by CNFuzzDD\footnote{\url{http://fmv.jku.at/cnfuzzdd/}} that each took at least a second to solve with the default configuration. In order to explore the extent to which it would matter if we optimized for one utility function but really cared about another, we identified the best configuration according to each of four utility functions (uniform, exponential, Pareto, log-Laplace) and then evaluated the quality of each configuration according to all of the utility functions. Our results (\cref{fig:cross_util}) show that these differences were significant in practice: we often lost a substantial fraction of the available utility when we optimized for the wrong utility function.

\paragraph{International SAT Competition.} \cref{fig:satresults} shows the ranking of the Parallel Track of the 2021 International SAT Competition.\footnote{The Parallel Track had fewer entrants than other tracks, yielding an easier-to-read figure. Other tracks exhibit qualitatively similar responses to changing the utility function used.} As discussed in \cref{sec:intro}, the organizers of these competitions have recognized that there are different reasonable choices for evaluation metrics (e.g., do we reward algorithms that are likely to finish, or those that are fast when they do finish), and that different choices will lead to different rankings of solvers. \cref{fig:satresults} ranks each solver according to its average score on the competition's set of test instances. No two rankings are exactly alike; each utility function gives a different top-three ranking. It is important to note that there is no objective ``best'' algorithm. Each of the entrants has its own merits, having been carefully developed by individuals who genuinely believed it to be a good solver. How good depends substantially on our metric for ``good'', as \cref{fig:satresults} demonstrates.

\section{Conclusion}\label{sec:conclusion}

It is nontrivial to identify a general scoring function that reflects our preferences over runtime distributions, particularly given that we often observe only capped runtimes. Such functions are needed whenever algorithms are compared based on their performance, particularly when these comparisons are done programatically, as in the case of automated algorithm design. Following \citet{vonneumann1944theory}, this paper has identified the constrained family of utility functions that results when our preferences obey six simple axioms. We have worked through a wide range of examples showing utility functions that would be appropriate in different scenarios, both when captime distributions are known and when they are underspecified; in the latter case, we appeal to the method of maximum entropy. Finally, we have shown that, given a specific utility function, we can estimate an algorithm's score to within a desired degree of accuracy using capped runtime samples, where the size of the captime depends on the desired accuracy and on the utility function.

\section*{Acknowledgements}

The first two authors were funded by an NSERC Discovery Grant, a DND/NSERC Discovery Grant Supplement, a CIFAR Canada AI Research Chair (Alberta Machine Intelligence Institute), a Compute Canada RAC Allocation, awards from Facebook Research and Amazon Research, and DARPA award FA8750-19-2-0222, CFDA \#12.910 (Air Force Research Laboratory). The research of the third author was supported in part by NSF awards CCF-2006737 and CNS-2212745. 

\bibliography{refs}
\bibliographystyle{icml2023}

%%%%%%%%%%%%%%%%%%%%%%%%%%%%%%%%%%%%%%%%%%%%%%%%%%%%%%%%%%%%%%%%%%%%%%%%%%%%%%%
%%%%%%%%%%%%%%%%%%%%%%%%%%%%%%%%%%%%%%%%%%%%%%%%%%%%%%%%%%%%%%%%%%%%%%%%%%%%%%%
% APPENDIX
%%%%%%%%%%%%%%%%%%%%%%%%%%%%%%%%%%%%%%%%%%%%%%%%%%%%%%%%%%%%%%%%%%%%%%%%%%%%%%%
%%%%%%%%%%%%%%%%%%%%%%%%%%%%%%%%%%%%%%%%%%%%%%%%%%%%%%%%%%%%%%%%%%%%%%%%%%%%%%%
\newpage
\appendix
\onecolumn

\section{Proofs of Theorems}\label{app:proofs} 

\paragraph{\cref{thm:main}:} If our preferences follow \cref{axiom:relevance,axiom:transitivity,axiom:continuity,axiom:monotonicity,axiom:independence,axiom:eagerness}, then a function $u$ satisfies 
    \begin{align}
        A \;\prefgteq{K}\; B  \iff \E_{t \sim A, \kappa \sim K}\big[u(t , \kappa)\big] \ge \E_{t \sim B, \kappa \sim K}\big[u(t , \kappa)\big],
    \end{align}
    for any runtime distributions $A$ and $B$ and any timeout distribution $K$ if and only if there are constants $c_0$ and $c_1 > 0$ such that $u(t, \kappa) = c_1 p(t, \kappa) + c_0$. Furthermore, $p$ has the form
    \begin{enumerate}
        \item $p(0, \kappa) = 1$ (maximum achieved at $t=0$),
        \item $p(t, \kappa) \ge p(t', \kappa)$ for all $t \le t'$ (monotonically decreasing),
        \item $p(t, \kappa) > 0$ for all $t < \kappa$ (strictly positive),
        \item $p(\kappa, \kappa) = 0$ (minimum achieved at $t=\kappa$).
    \end{enumerate}

\begin{proof}\label{pf:thm:main}
    Given an arbitrary runtime distribution $A$ and a timeout distribution $K$, we will construct a new synthetic algorithm $X$ that returns an answer either instantaneously or after some amount of time sampled from $K$. Formally, 
    \begin{align}\label{eqn:pf:thm:main:x1}
        X = \bigg[ \Big[ p(t, \kappa) :  \delta_0 \,,\, \big(1 - p(t, \kappa) \big) : \delta_{\infty} \Big] \;\big|\; t \sim A, \kappa \sim K \bigg],
    \end{align}
    where $p$ is defined in \cref{def:p}. Setting 
    \begin{align*}
    p_A = \int_0^{\infty} \int_0^{\infty} p(t, \kappa) dF_A(t) dF_{K}(\kappa) = \E_{t \sim A, \kappa \sim K}\big[ p(t, \kappa) \big],
    \end{align*}
    we can write $X$'s runtime  distribution as
    \begin{align}\label{eqn:pf:thm:main:x2}
        X \;=\; \big[ p_A : \delta_0 \,,\, (1 - p_A) : \delta_{\infty} \big].
    \end{align}
    
    Consider the function $M(t, \kappa) = [p(t, \kappa) :  \delta_0 \,,\, (1 - p(t, \kappa) ) : \delta_{\kappa}]$ that maps runtime--captime pairs to mixture distributions. Since $p$ was defined in \cref{def:p} so that $\delta_{t} \prefeq{\delta_{\kappa}} M(t, \kappa)$, we can conclude from Independence that 
    \begin{align}\label{eqn:pf:thm:main:akx}
        A \prefeq{K} \bigg[ \Big[ p(t, \kappa) :  \delta_0 \,,\, \big(1 - p(t, \kappa) \big) : \delta_{\infty} \Big] \;\big|\; t \sim A, \kappa \sim K \bigg].
    \end{align}
    \cref{eqn:pf:thm:main:x1,eqn:pf:thm:main:x2,eqn:pf:thm:main:akx} together then give that 
        \begin{align}
            A \prefeq{K} \big[ p_A : \delta_0 \,,\, (1 - p_A) : \delta_{\infty} \big].
        \end{align}
    Now consider a second algorithm $B$, and define $Y$ and $p_B$ analogously to $X$ and $p_A$, but with $B$ in place of $A$, so that by the same argument we have
        \begin{align}
            B \prefeq{K} \big[ p_B : \delta_0 \,,\, (1 - p_B) : \delta_{\infty} \big].
        \end{align}
    Since $\delta_0 \prefgteq{K} \delta_{\infty}$ by Eagerness, Monotonicity tells us that $\big[ p_A : \delta_0 \,,\, (1 - p_A) : \delta_{\infty} \big] \prefgteq{K} \big[ p_B : \delta_0 \,,\, (1 - p_B) : \delta_{\infty} \big]$ iff $p_A \ge p_B$, and thus
    \begin{align}\label{eqn:pf:thm:main:putil}
        A \prefgteq{K} B \iff \E_{t \sim A, \kappa \sim K} \big[ p(t, \kappa) \big] \ge \E_{t \sim B, \kappa \sim K} \big[ p(t, \kappa) \big] .
    \end{align}
    So the function $p$ can serve as a utility function, and we can use the biconditional \cref{eqn:pf:thm:main:putil} to infer certain aspects of $p$'s form: 
    
    \begin{enumerate}
        \item By definition of $p$ we have $\delta_0 \prefeq{\delta_{\kappa}} \big[ p(0, \kappa) : \delta_0 \,,\, (1 - p(0, \kappa)) : \delta_{\infty} \big]$, where $p(0, \kappa) \le 1$, and by Eagerness we have $\delta_0 \prefgteq{\delta_{\kappa}} \delta_{\infty}$ so applying Monotonicity with $A=\delta_0, B=\delta_{\infty}$ and $q = 1$, we have that $p(0, \kappa) \ge 1$, and thus $p(0, \kappa) = 1$. 
        \item For any $t \le t' < \kappa$, Eagerness tells us that $\delta_t \prefgteq{\delta_{\kappa}} \delta_{t'}$, and so $p(t, \kappa) \ge p(t', \kappa)$.
        \item Relevance states that $\delta_t \prefgt{\delta_{\kappa}} \delta_{\kappa}$, and so $p(t, \kappa) > p(\kappa, \kappa) = 0$.
        \item By definition, $p(\kappa, \kappa)$ is set to $0$. 
    \end{enumerate}
    
    Together this all means that the function $p$ is monotonically decreasing from $1$ to $0$, reaching $0$ only at $t = \kappa$ (i.e., it is strictly greater than $0$ for all $t < \kappa$). So $p$ has the given form and can serve as a utility function. 
    
    We can now show that a function $u$ satisfies \cref{eqn:thm:main:uutil} if and only if it has the form $u(t, \kappa) = c_1 p(t, \kappa) + c_0$ for some $c_1> 0$ and $c_0$. The reverse, `only if' direction follows immediately from linearity of expectation. For the forward, `if' direction, suppose that $u$ does satisfy \cref{eqn:thm:main:uutil} for all $A, B$ and $K$. Since $\delta_0 \prefgteq{K} A \prefgteq{K} \delta_\infty$ by Eagerness, Continuity says that there exists a constant $\alpha$ such that $A \prefeq{K} \big[ \alpha : \delta_0 , (1 - \alpha) : \delta_{\infty} \big]$. Using this equivalence, \cref{eqn:pf:thm:main:putil} tells us that
    \begin{align*}
        \E_{t \sim A, \kappa \sim K} \big[ p(t, \kappa) \big] = \alpha,
    \end{align*}
    and applying \cref{eqn:thm:main:uutil} to $A$ and the equivalent mixture $\big[ \alpha : \delta_0 , (1 - \alpha) : \delta_{\infty} \big]$ tells us that 
    \begin{align*}
        \E_{t \sim A, \kappa \sim K} \big[ u(t, \kappa) \big] &= \alpha \E_{\kappa \sim K} \big[ u(0, \kappa) \big] + (1 - \alpha) \E_{\kappa \sim K} \big[ u(\infty, \kappa) \big] \\
        &= \alpha \big( \E_{\kappa \sim K} \big[ u(0, \kappa) - u(\infty, \kappa) \big] \big) + \E_{\kappa \sim K} \big[ u(\infty, \kappa) \big],
    \end{align*}
    so setting $c_1= \E_{\kappa \sim K} \big[ u(0, \kappa) - u(\infty, \kappa) \big]$ and $c_0 = \E_{\kappa \sim K} \big[ u(\infty, \kappa) \big]$ we have that
    \begin{align*}
        \E_{t \sim A, \kappa \sim K} \big[ u(t, \kappa) \big] = c_1 \E_{t \sim A, \kappa \sim K} \big[ p(t, \kappa) \big] + c_0
    \end{align*}
    for any $A$ and $K$. In particular, when $A = \delta_t$ and $K = \delta_{\kappa}$ for arbitrary $t$ and $\kappa$ we get that $u(t, \kappa) = c_1 p(t, \kappa) + c_0$, which completes the proof.
\end{proof}

\paragraph{\cref{thm:timebound}:} The time to estimate the score of an arbitrary algorithm $A$ to within an $\epsilon$ additive factor with probability $1 - \delta$ will be greater than $m \cdot u^{-1}(2\epsilon)$ in the worst case, but always less than $m \cdot u^{-1}(\epsilon/2)$, where $m = \frac{\ln(2/\delta)}{2}\big(\frac{2-\epsilon}{\epsilon}\big)^2$.

    The proof follows from the next two lemmas. The number of samples required is determined by Hoeffding's inequality, independent of the choice of captime. If we are constrained by an accuracy parameter $\epsilon$, then we must do runs at an implied captime of $u^{-1}(\epsilon / 2)$. If we are constrained by a captime $\kappa$, then we must do enough runs to apply Hoeffding's inequality with an accuracy parameter of $u(\kappa)$. 
    
    \begin{lemma}\label{lem:mbound}
        For any $\epsilon, \delta$, any algorithm $A$, and any utility function $u$, if we take $m = \frac{\ln(2/\delta)}{2}\big(\frac{2-\epsilon}{\epsilon}\big)^2$ runtime samples from $A$ at captime $\kappa = u^{-1}(\epsilon/2)$, then the capped sample mean utility will be within $\epsilon$ of the true, uncapped mean utility with probability at least $1 - \delta$.
    \end{lemma}
    \begin{proof}
        By the triangle inequality 
        \begin{align*}
            &\Pr\bigg( \bigg| \frac{1}{m}\sum_{j=1}^m u(t_j(\kappa)) - \Emath_{t \sim A}\big[ u(t) \big] \bigg| \ge \epsilon \bigg) \\
            &= \Pr\bigg( \bigg| \frac{1}{m}\sum_{j=1}^m u(t_j(\kappa)) - \Emath_{t \sim A}\big[ u(t(\kappa)) \big] + \Emath_{t \sim A}\big[ u(t(\kappa)) - u(t) \big] \bigg| \ge \epsilon \bigg) \\
            & \le \Pr\bigg( \bigg| \frac{1}{m}\sum_{j=1}^m u(t_j(\kappa)) - \Emath_{t \sim A}\big[ u(t(\kappa)) \big] \bigg| + \bigg| \Emath_{t \sim A}\big[ u(t(\kappa)) - u(t) \big] \bigg| \ge \epsilon \bigg).
        \end{align*}
        But since $\E_{t \sim A}\big[ u(t(\kappa)) - u(t) \big] = \E_{t \sim A}\big[ u(t(\kappa)) - u(t) \big| t \ge \kappa \big] \Pr_{t \sim A}\big( t \ge \kappa \big)$, and $u(t(\kappa)) = u(\kappa) = \epsilon / 2$ for $t \ge \kappa$, we have that 
            \begin{align*}
            \Emath_{t \sim A}\big[ u(t(\kappa)) - u(t) \big] &\le \frac{\epsilon}{2}.
        \end{align*}
        Together these tell us that 
        \begin{align*}
            \Pr\bigg( \bigg| \frac{1}{m}\sum_{j=1}^m u(t_j(\kappa)) - \Emath_{t \sim A}\big[ u(t) \big] \bigg| \ge \epsilon \bigg) &\le \Pr\bigg( \bigg| \frac{1}{m}\sum_{j=1}^m u(t_j(\kappa)) - \Emath_{t \sim A}\big[ u(t(\kappa)) \big] \bigg| \ge \frac{\epsilon}{2} \bigg). 
        \end{align*}
        Then using the fact that $u(t_j(\kappa)) \in [\frac{\epsilon}{2}, 1]$ for all $j$, Hoeffding's inequality tells us that 
        \begin{align*}
            \Pr\bigg( \bigg| \frac{1}{m}\sum_{j=1}^m u(t_j(\kappa)) - \Emath_{t \sim A}\big[ u(t) \big] \bigg| \ge \epsilon \bigg) \le \delta
        \end{align*}
        if we take at least $m = \frac{\ln(2/\delta)}{2}\big(\frac{2-\epsilon}{\epsilon}\big)^2$ capped samples. 
    \end{proof}
    
    \cref{lem:mbound} shows that if we take enough samples at a large enough captime, then we can accurately estimate any distribution's mean utility. We know that if we take too few samples we will never be able to estimate a distribution well (even if they are uncapped samples). The next lemma shows it is also true that no matter how many samples we take, if the captime we take them at is too small, then we will fail to estimate some distributions well. 
    
    \begin{lemma}\label{lem:kbound}
        For any $\epsilon$ and any utility function $u$, there exists a distribution $A$ such that no matter how many samples we take, if the captime $\kappa < u^{-1}(2\epsilon)$, then the capped sample mean utility will be at least $\epsilon$ from the true, uncapped mean.
    \end{lemma}
    \begin{proof}
        By the (reverse) triangle inequality, and since utilities are positive, we have 
        \begin{align*}
            \bigg| \frac{1}{m}\sum_{j=1}^m u(t_j(\kappa)) - \Emath_{t \sim A}\big[ u(t) \big] \bigg| &\ge \frac{1}{m}\sum_{j=1}^m u(t_j(\kappa)) - \Emath_{t \sim A}\big[ u(t) \big].
        \end{align*}
        Since $t_j(\kappa) \le \kappa < u^{-1}(2\epsilon)$ for all $j$, we have $u(t_j(\kappa)) \ge u(\kappa) > 2\epsilon$, and so
        \begin{align*}
            \frac{1}{m}\sum_{j=1}^m u(t_j(\kappa)) > 2\epsilon .
        \end{align*}
        Since $u(t) \to 0$ as $t \to \infty$, there exists some $t_{\epsilon}$ such that $u(t_{\epsilon}) < \epsilon$. If we choose $A$ to always return $t_{\epsilon}$, we have that 
        \begin{align*}
            \Emath_{t \sim A}\big[ u(t) \big] < \epsilon.
        \end{align*}
        Combining the above we have that 
        \begin{align*}
            \bigg| \frac{1}{m}\sum_{j=1}^m u(t_j(\kappa)) - \Emath_{t \sim A}\big[ u(t) \big] \bigg| &\ge \epsilon
        \end{align*}
        regardless of how many samples are taken. 
    \end{proof}

\section{Independence, Decomposability, and the Compounding of Preferences}\label{app:independence}

The Independence axiom plays an important role in the proof of our main theorem, and in expected-utility theory in general. It is this axiom that makes our scoring function an expectation. 

The Independence axiom was not one of the axioms stated by \citet{vonneumann1944theory}, but emerged from a frenzy of activity that followed the publication of their result. \citet{malinvaud1952note} revealed that in fact their framework \emph{did} contain an Independence assumption, it was just hidden in their formal setup. We discuss this further below, with specific reference to our own setting of runtime distributions. Whatever its origins in decision theory, the mathematics of the Independence axiom were explored at least by \citet{kolmogorov1930notion}, in whose condition (iv) can be found the essence of the axiom. For detailed histories of Independence, its development, and incorporation into expected utility theory see, e.g., \citet{fishburn1989retrospective,fishburn1995invention,moscati2016retrospectives}

Different sources use different notation and slightly different forms of the Independence axiom, which can obscure the fact that they are all really saying the same thing, which is also what our own Independence axiom is saying: when we choose between two gambles, we can focus on the places those two gambles differ, because this is the part that will determine our preferences. We present a handful of different versions here, just to give an idea of how these are similar to our own. 

\textbf{Axiom II (Strong Independence)} \citep{samuelson1952probability}: \textit{If lottery ticket $(A)_1$ is (as good or) better than $(B)_1$, and lottery ticket $(A)_2$ is (as good or) better than $(B)_2$, then an even chance of getting $(A)_1$ or $(A)_2$ is (as good or) better than an even chance of getting $(B)_1$ or $(B)_2$.}

\textbf{Condition 2: Independence} \citep{fishburn1970utility}: $(P \prefgt{} Q, 0 < \alpha < 1) \Longrightarrow \alpha P + (1- \alpha) R \prefgt{} \alpha Q + (1- \alpha) R$.

\textbf{P2: sure-thing principle} \citep{savage1972foundations}: If $\mathbf{f}$, $\mathbf{g}$, and $\mathbf{f'}$,$\mathbf{g'}$ are such that:

\hspace{1cm} 1. in $\sim B$, $\mathbf{f}$ agrees with $\mathbf{g}$, and $\mathbf{f'}$ agrees with $\mathbf{g'}$,

\hspace{1cm} 2. in $B$, $\mathbf{f}$ agrees with $\mathbf{f'}$, and $\mathbf{g}$ agrees with $\mathbf{g'},$

\hspace{1cm} 3. $\mathbf{f} \le \mathbf{g}$;

then $\mathbf{f'} \le \mathbf{g'}$.

\textbf{Axiom 3.1.3 (Substitutability)} \citep{shoham2008multiagent}: \textit{If $o_1 \sim o_2$, then for all sequences of one or more outcomes $o_3,...,o_k$ and sets of probabilities $p, p_3,..., p_k$ for which $p + \sum_{i=3}^k p_i = 1$, $[p:o_1, p_3:o_3,...,p_k:o_k] \sim [p:o_2, p_3:o_3,...,p_k:o_k]$.} 

\textbf{NM2 \textit{Independence}} \citep{parmigiani2009decision}: for every $a$, $a'$, and $a''$ in $\mathcal{A}$ and $\alpha \in (0,1]$, we have 
\begin{align*}
    a \prefgt{} a' \quad \text{implies} \quad (1-\alpha)a'' + \alpha a \prefgt{} (1-\alpha)a'' + \alpha a'. 
\end{align*} 

\textbf{Axiom 3.13. Independence} \citep{bacci2019introduction}: \textit{Given $c_i, c_j, c_h \in C$ such that $c_i \sim c_j$, then $\langle c_i p c_h \rangle \sim \langle c_j p c_h \rangle$}.

We can see that all of these statements are really saying the same thing. Our preferences between gambles are determined by our preferences for their component parts; or put another way, our preferences for gambles are independent of the parts of those gambles that are the same. We can understand our own Independence axiom of \cref{sec:axioms} in the context of these others. We can trivially write any distribution $A$ as $[\delta_t \st t \sim A, \kappa \sim K]$ and if we note that $A \prefgteq{K} B$ is really just shorthand for $A \times K \prefgteq{} B \times K$, our Independence axiom can be stated as 

``if 
\begin{align}\label{eqn:component}
    \delta_t \times \delta_{\kappa} \;\prefeq{}\; M(t, \kappa) \times \delta_{\kappa}  
\end{align}
for all $t, \kappa$, then 
\begin{align}\label{eqn:compound}
    [\delta_t \st t \sim A, \kappa \sim K] \times K \;\prefeq{}\; [M(t, \kappa) \st t \sim A, \kappa \sim K] \times K \text{''.}
\end{align}

This says that our preferences for compound distributions (\cref{eqn:compound}) are determined by our preferences for their individual components (\cref{eqn:component}). Taking \citep{samuelson1952probability} as an example, we can write the ``Strong Indedpendence'' axiom in our notation as:

``if 
\begin{align*}
    A_1 \prefgteq{} B_1 \quad \text{and} \quad A_2 \prefgteq{} B_2
\end{align*}
then
\begin{align*}
    [p: A_1, (1-p): A_2] \prefgteq{} [p: B_1, (1-p): B_2] \text{''.}
\end{align*}
We can stress the similarity between the above and our version, by noting the irrelevant differences: We use indifference while Samuelson uses weak preference. Our compound distributions are mixtures over all $t,\kappa$, while Samuelson's are mixtures over just the indices $1,2$. Samuelson's preferences are over arbitrary gambles $A_1, B_1, A_2, B_2$, while our indifferences are stated only for sure outcomes $\delta_t$ with mixtures $M(t, \kappa)$ between the best and worst outcome, under sure captimes $\delta_{\kappa}$. Finally, and perhaps most significantly, the use of product distributions in our Independence axiom means that captimes compound in a somewhat non-obvious way, which essentially amounts to an assumption that any two draws from the captime distribution are the same to us, given the runtime $t$. 

Independence does something important that no other VNM axiom does: it makes a statement about the way our preferences ``carry over'' when distributions are nested within each other. If we prefer one distribution to another, how would we feel about the result of each of these nested within some third distribution? Monotonicity and Continuity make statements about nested distributions, but they do not say anything about how our preferences for distinct components extend to our preferences for the distributions they are part of. That is the role that Independence plays. It says that our preferences for complex, nested outcomes are determined by our preferences for their simple components.

\subsection{Violations of the Independence Axiom}

Violations of the Independence Axiom are well-known \citep{allais1953comportement,ellsberg1961risk,slovic1974accepts,maccrimmon1979utility,kahneman1979prospect}. When presented with simple choices between gambles, seemingly rational individuals make apparently reasonable selections that are not consistent with the Independence axiom. Some decision theorists have been quick to dismiss these choices as ``wrong'', stressing the rationality of the Independence axiom and thus the irrationality of violating it. Others have been happy enough to just discard Independence.

The validity of the Independence axiom is not the subject of this paper. We claim only that \emph{if} a decision-maker's choices follow the axioms, \emph{then}\footnote{In fact, the implication is bidirectional, but we are less interested in the other direction.} they will act as though they are choosing algorithms according to a utility function of the given form. We note, however, that our setting is somewhat unique among decision problems in that \emph{we do not have to incur the full loss of a bad outcome}. Since many of the observed violations of Independence can be understood as individuals trying to avoid worst-case outcomes (i.e., they do not want to end up with nothing when they had a decent chance of getting something), there is good reason to suppose that it does hold in our particular setting. Implicit in the assumption that we face a captime is the assumption that there is some default action available to us, and because we can always choose to terminate the algorithm ourselves, we can take this option at any time. This is a unique characteristic of our setting. 

Consider an algorithm that either finishes in 1 second or in 100 years. If we use this algorithm, we will likely decide that it is not worth waiting 100 years for the solution in those cases where it does not finish in 1 second, and in practice never run for much more than, say, 2 seconds. Now consider an investment with analogous monetary payouts. If say, we either gain \$1 million or lose \$10 million we cannot simply decide in the bad case that we have lost too much money and take some other, smaller loss. We cannot simply ``stop'' the investment the way we can stop an algorithm run (at least not if we want to keep investing). We can only hedge against the bad case \emph{before} we observe the outcome by buying an option on the same asset, for instance. In the algorithm runtime case, we can decide to limit our losses \emph{while we observe the outcome}, as if we had actually purchased some sort of ``hedge'' option on our ``runtime investment.'' In effect, we never have to accept large losses. 

Finally, we note that whatever psychological effects a decision-maker may experience when choosing between complex, nested gambles in general, in the case of algorithm runtime distributions it is especially reasonable to suppose that differences in the structure of randomness do not matter to us, since they are essentially hidden from us in practice. When we run our algorithm, we generally will not be, and certainly need not be aware of its inner workings in any detailed way. We are not explicitly exposed to the structure of the randomness, only the final outcome. An algorithm that either runs some subroutine $A$ or some subroutine $B$ with a 50/50 chance depending on its random seed does not appear to us as a coin toss followed by a draw from one of $A$ or $B$. It simply returns an answer to us after some elapsed time. Whatever the case may be in other settings, we really are only concerned with the distribution of final outcomes, because that is really all we observe.

\subsection{Decomposability}

\citet{vonneumann1944theory} established a kind of preference-nesting that is related to, but different from, Independence with an axiom they called the ``algebra of combining.'' Stated in our notation this would be
\begin{small}
\begin{align}\label{eqn:vnmalgebra}
    \Big[p: \big[ q: A, (1-q):B \big], (1-p):B \Big] = \Big[pq: A, (1-pq):B \Big],
\end{align}
\end{small}
which simply says that the probabilities of nested distributions over outcomes obey the normal rules of probability (i.e., they \emph{compound} in the normal way). Note, however, the equality in \cref{eqn:vnmalgebra}. To von Neumann and Morgenstern this was a \emph{true equality}, not an indifference. In fact, they did not axiomatize an indifference relation at all. In effect, their outcome space was a set of indifference classes, the elements of each class being represented by what they called an ``abstract utility.'' In this way, they do not talk about specific outcomes, or even about distributions over specific outcomes, they simply define the operation $[p:A, (1-p):B]$, where $A$ and $B$ are ``abstract utilities'', and state that the outcome space is closed under this operation. In this way, their mixture operation is not a probability distribution \emph{per se}, and an axiom like \cref{eqn:vnmalgebra} is necessary to establish the rules for manipulating such expressions (literally, to establish their algebra). What's more, since $A$ and $B$ are equivalence classes, \cref{eqn:vnmalgebra} has a more subtle interpretation than it might first appear. In particular, when we note that $A$ and $B$ are actually sets of (distributions over) outcomes, we can see why an axiom describing this combination process becomes necessary. 

In our setting of runtime distributions, the equality in \cref{eqn:vnmalgebra} follows immediately from the mathematical fact that distributions can be combined to form new distributions using the normal compounding operation. Hence, no such axiom is needed. Some later authors have axiomatized this property, calling it \emph{Decomposability}, and replacing the equality with an indifference. This is not strictly necessary and others have simply taken for granted, correctly, that this operation can be performed. A distribution over distributions is a distribution and no preference can change this fact. Thus, if $A$ and $B$ are (distributions over) fixed outcomes, and if we interpret $[p:A, (1-p):B]$ not as an abstract operation like von Neumann and Morgenstern did, but as a lottery that gives (a draw from) $A$ with probability $p$ and (a draw from) $B$ with probability $1-p$, then the equality in \cref{eqn:vnmalgebra} is a true equality, regardless of what the outcome space is. This makes no assertion about how such nesting affects our preferences, only that such nesting is possible, and that the result is a valid object for which preferences can be defined. The assertion of how such nesting affects our preferences is made by the Independence axiom. 

Independence was shown by \citet{malinvaud1952note} to be implicit in Von Neumann and Morgenstern's use of indifference classes. We can understand this in the language of our algorithm runtime setting as follows. 

The ``abstract utilities'' $U$ and $V$ are equivalence classes of algorithms (i.e., indifference sets of algorithms). We are ``indifferent'' between two algorithms if and only if they belong to the same equivalence class. As noted above, von Neumann and Morgenstern define an abstract operation $[p: U, (1-p): V] = W$ on equivalence classes. The resulting equivalence class $W$ is understood to be the set of all algorithm runtime distributions of the form $[p:A, (1-p):B]$ where $A \in U$ and $B \in V$. Thus, if $A_1, B_1 \in U$ and $A_2, B_2 \in V$, then $[p: A_1 , (1-p): A_2] \in W$ and $[p: B_1 , (1-p): B_2] \in W$. Or saying the same thing in the language of indifference we have: if $A_1 \prefeq{} B_1$ and $A_2 \prefeq{} B_2$, then $[p: A_1 , (1-p): A_2] \prefeq{} [p: B_1 , (1-p): B_2]$, which is a form of the Independence axiom.

%%%%%%%%%%%%%%%%%%%%%%%%%%%%%%%%%%%%%%%%%%%%%%%%%%%%%%%%%%%%%%%%%%%%%%%%%

\section{Maximum Entropy Distribution Derivations}\label{app:maxent}

Entropy represents the average number of bits required to specify the value of an outcome. When that value is a continuous quantity, as in the case of time, entropy is technically infinite. However, we are actually never interested in entropy in an absolute sense, only relative differences in entropy matter. The continuous analogue of entropy, and the quantity we maximize, is referred to as \emph{differential} entropy (of the distribution $f$), and is given by 
\begin{align*}
    h\big[f\big] = - \int_0^{\infty} f(\kappa) \log f(\kappa)  \;d\kappa .
\end{align*}
Informally, the entropy of a continuous distribution $f$ approaches $h[f] + \infty$ as it is represented with increasing precision, so differences in differential entropy are the same as differences in Shannon entropy, since the $\infty$ ``cancels'' by subtraction. This justifies the use of differential entropy as a stand-in for Shannon entropy. In real implementations, the entropy of an $n$-bit quantization of a continuous quantity is approximated by $h[f] + n$. See e.g., \cite{cover2006elements} for details.

Derivation of the maximum entropy distributions is done using the calculus of variations and the method of Lagrange multipliers.

\subsection{Bounded (uniform).} 
Timeout is bounded in $[0, \kappa_0]$.

The functional is 
\begin{align*}
    L[f(\kappa), \lambda] = - \int_0^{\kappa_0} f(\kappa) \log (f(\kappa)) \;d \kappa - \lambda_0 \bigg( \int_0^{\kappa_0} f(\kappa) \;d\kappa - 1 \bigg) .
\end{align*}
The partial derivatives are 
\begin{align*}
    \frac{\partial L}{\partial f(\kappa)} &= - \log (f(\kappa)) - 1 - \lambda_0 \\
    \frac{\partial L}{\partial \lambda_0} &= - \int_0^{\kappa_0} f(\kappa) \;d\kappa + 1 
\end{align*}
giving 
\begin{align*}
    f(\kappa) = \exp{(- \lambda_0 - 1)}
\end{align*}
with condition
\begin{align*}
    \int_0^{\kappa_0} f(\kappa) \;d\kappa = 1 .
\end{align*}
Solving for $\lambda_0$
gives 
\begin{align*}
    \lambda_0 = \log(\kappa_0) - 1
\end{align*}
and 
\begin{align*}
    f(\kappa) = 1 / \kappa_0
\end{align*}
and 
\begin{align*}
    F(t) = t / \kappa_0.
\end{align*}

\subsection{Fixed Expectation (exponential).} 
Timeout has expectation $\mu$.

The functional is 
\begin{align*}
    L[f(\kappa), \lambda] = - \int_0^{\infty} f(\kappa) \log (f(\kappa)) \;d \kappa &-\lambda_0 \bigg( \int_0^{\infty} f(\kappa) \;d\kappa - 1 \bigg) -\lambda_1 \bigg( \int_0^{\infty} \kappa f(\kappa) \;d\kappa - \mu \bigg) .
\end{align*}
The partial derivatives are 
\begin{align*}
    \frac{\partial L}{\partial f(\kappa)} &= - \log (f(\kappa)) - 1 - \lambda_0 - \lambda_1 \kappa \\
    \frac{\partial L}{\partial \lambda_0} &= - \int_0^{\infty} f(\kappa) \;d\kappa + 1 \\
    \frac{\partial L}{\partial \lambda_1} &= - \int_0^{\infty} \kappa f(\kappa) \;d\kappa + \mu
\end{align*}
giving 
\begin{align*}
    f(\kappa) = \exp{(- \lambda_1 \kappa - \lambda_0 - 1)}
\end{align*}
with conditions, 
\begin{align*}
    \int_0^{\infty} f(\kappa) \;d\kappa &= 1 \\
    \int_0^{\infty} \kappa f(\kappa) \;d\kappa &= \mu.
\end{align*}
Solving for $\lambda_0$ and $\lambda_1$
gives 
\begin{align*}
    \lambda_0 &= \log( 1 / \lambda_1 ) - 1 \\
    \lambda_1 &= 1 / \mu\\
\end{align*}
and 
\begin{align*}
    f(\kappa) = \frac{1}{\mu} \exp{\bigg(-\frac{\kappa}{\mu}\bigg)}
\end{align*}

\subsection{Fixed order of magnitude (Pareto).} 
Timeout has $\E_{\kappa_\sim K}[\log(\kappa/\kappa_0)] = 1/\alpha$ with $\kappa > \kappa_0$.

The functional is 
\begin{align*}
    L[f(\kappa), \lambda] = - \int_{\kappa_0}^{\infty} f(\kappa) \log (f(\kappa)) \;d \kappa &-\lambda_0 \bigg( \int_{\kappa_0}^{\infty} f(\kappa) \;d\kappa - 1 \bigg) - \lambda_1 \bigg( \int_{\kappa_0}^{\infty} \log(\kappa / \kappa_0) f(\kappa) \;d\kappa - \frac{1}{\alpha} \bigg).
\end{align*}
The partial derivatives are 
\begin{align*}
    \frac{\partial L}{\partial f(\kappa)} &= - \log (f(\kappa)) - 1 - \lambda_0 - \lambda_1 \log(\kappa / \kappa_0) \\
    \frac{\partial L}{\partial \lambda_0} &= - \int_{\kappa_0}^{\infty} f(\kappa) \;d\kappa + 1 \\
    \frac{\partial L}{\partial \lambda_1} &= - \int_{\kappa_0}^{\infty} \log(\kappa / \kappa_0) f(\kappa) \;d\kappa + \frac{1}{\alpha}
\end{align*}
giving 
\begin{align*}
    f(\kappa) = \exp{(- \lambda_0 - 1)} \bigg( \frac{\kappa}{\kappa_0} \bigg)^{-\lambda_1}
\end{align*}
with conditions, 
\begin{align}
    \int_{\kappa_0}^{\infty} f(\kappa) \;d\kappa &= 1 \\
    \int_{\kappa_0}^{\infty} \log(\kappa / \kappa_0) f(\kappa) \;d\kappa &= \frac{1}{\alpha}.
\end{align}
Integrating (1) gives
\begin{align*}
    \exp(-\lambda_0 - 1) &= \frac{\lambda_1 - 1}{\kappa_0} \\
\end{align*}
so 
\begin{align*}
    f(\kappa) = \frac{(\lambda_1 - 1)\kappa_0^{\lambda_1 - 1}}{\kappa^{\lambda_1}},
\end{align*}
which is a Pareto pdf with parameter $\lambda_1 - 1$. Integrating\textbf{} gives that $\lambda_1 - 1 = \alpha$.

\subsection{Two-tailed fixed order of magnitude, (un)equal tails ((generalized) log-Laplace).} 
The conditions are $\E_{\kappa_\sim K}[ \log(\kappa_0/\kappa) \st \kappa < \kappa_0] = 1/\beta$ and ${\E_{\kappa_\sim K}[ \log(\kappa/\kappa_0) \st \kappa \ge \kappa_0]} = 1/\alpha$, with $\Pr_{\kappa \sim K}(\kappa < \kappa_0) = p$.
The functional is 
\begin{align*}
    L[f(\kappa), \lambda] = - \int_{0}^{\infty} f(\kappa) \log (f(\kappa)) \;d \kappa &- \lambda_0 \bigg( \int_{0}^{\infty} f(\kappa) \;d\kappa - 1 \bigg) \\
    &-\lambda_1 \bigg( \int_0^{\kappa_0} f(\kappa) - p \bigg) \\
    &- \lambda_2 \bigg( \int_{0}^{\kappa_0} \log(\kappa_0 / \kappa) f(\kappa) \;d\kappa  - \frac{p}{\beta} \bigg) \\
    &- \lambda_3 \bigg( \int_{\kappa_0}^{\infty} \log(\kappa / \kappa_0) f(\kappa) \;d\kappa - \frac{1-p}{\alpha} \bigg).
\end{align*}
The main partial derivative is
\begin{align*}
    \frac{\partial L}{\partial f(\kappa)} = - \log (f(\kappa)) - 1 - \lambda_0 &- \lambda_1 \ind(\kappa < \kappa_0) - \lambda_2 \log(\kappa_0 / \kappa) \ind(\kappa < \kappa_0) - \lambda_3 \log(\kappa / \kappa_0)\ind(\kappa \ge \kappa_0),
\end{align*}
giving the pdf: 
\begin{align}
    f(\kappa) = \exp\Big(-1 -\lambda_0 - \lambda_1 \ind(\kappa < \kappa_0)\Big) \bigg(\frac{\kappa}{\kappa_0}\bigg)^{\lambda_2 \ind(\kappa < \kappa_0) - \lambda_3 \ind(\kappa \ge \kappa_0)}
\end{align}
with the conditions 
\begin{align*}
    \int_{0}^{\infty} f(\kappa) \;d\kappa &= 1 \\
    \int_0^{\kappa_0} f(\kappa) &= p \\
    \int_{0}^{\kappa_0} \log(\kappa_0 / \kappa) f(\kappa) \;d\kappa &= \frac{p}{\beta} \\
    \int_{\kappa_0}^{\infty} \log(\kappa / \kappa_0) f(\kappa) \;d\kappa &= \frac{1-p}{\alpha},
\end{align*}
which give us that 
\begin{align*}
    \lambda_3 &= \alpha + 1 \\
    \lambda_2 &= \beta - 1 \\
    \exp(-\lambda_1) &= \frac{p \beta}{(1 - p) \alpha} \\
    \exp(- 1 - \lambda_0) &= \frac{(1- p) \alpha}{\kappa_0}
\end{align*}
and so 
\begin{align*}
    f(\kappa) = \begin{cases}
        \frac{p \beta}{\kappa_0}\Big(\frac{\kappa}{\kappa_0}\Big)^{\beta - 1} &\quad \text{if } \kappa < \kappa_0 \\
        \frac{(1-p) \alpha}{\kappa_0}\Big(\frac{\kappa_0}{\kappa}\Big)^{\alpha + 1} &\quad \text{otherwise}
    \end{cases} .
\end{align*}
Ensuring the continuity condition $\lim_{\kappa \to \kappa_0^{-}} f(\kappa) = f(\kappa_0)$ means that $p = \frac{\alpha}{\alpha + \beta}$, giving the generalized log-Laplace pdf. When $\alpha = \beta$, it becomes the standard log-Laplace.

\subsection{Fixed squared-deviation (log-normal)}
The conditions are $\E_{\kappa_\sim K}[ \log(\kappa/\kappa_0) ] = 0$ and $\E_{\kappa_\sim K}[ ( \log(\kappa/\kappa_0) )^2 ] = \sigma^2$. 
The functional is 
\begin{align*}
    L[f(\kappa), \lambda] = - \int_{0}^{\infty} f(\kappa) \log (f(\kappa)) \;d \kappa &- \lambda_0 \bigg( \int_{0}^{\infty} f(\kappa) \;d\kappa - 1 \bigg) \\
    &- \lambda_1 \bigg( \int_{0}^{\infty} \log(\kappa / \kappa_0) f(\kappa) \;d\kappa  \bigg) \\
    &- \lambda_2 \bigg( \int_{0}^{\infty} (\log(\kappa / \kappa_0))^2 f(\kappa) \;d\kappa  - \sigma^2 \bigg) .
\end{align*}
The main partial derivative is
\begin{align*}
    \frac{\partial L}{\partial f(\kappa)} &= - \log (f(\kappa)) - 1 - \lambda_0 - \lambda_1 \log(\kappa / \kappa_0) - \lambda_2 \big( \log(\kappa / \kappa_0) \big)^2
\end{align*}
giving the pdf: 
\begin{align}
    f(\kappa) = \exp\Big(-1 -\lambda_0\Big) \bigg(\frac{\kappa_0}{\kappa}\bigg)^{\lambda_1}\exp\bigg( -\lambda_2 \big( \log(\kappa / \kappa_0) \big)^2 \bigg),
\end{align}
which is a log-normal distribution when $\lambda_2 = \frac{1}{2\sigma^2}$, $\lambda_1 = 1$, and $\exp(-1 - \lambda_0) = \frac{1}{\kappa_0 \sigma \sqrt{2 \pi}}$. 

\subsection{Bounded support and left tail (peicewise).}
Timeout is bounded in $[0, \kappa_0]$ and $\Pr_{\kappa \sim K}(\kappa > \kappa_1) \ge 1 - \delta$.

The functional is 
\begin{align*}
    L[f(\kappa), \lambda] = - \int_0^{\kappa_0} f(\kappa) \log (f(\kappa)) \;d \kappa &- \lambda_0 \bigg( \int_0^{\kappa_0} f(\kappa) \;d\kappa - 1 \bigg) \\
    &- \lambda_1 \bigg( \int_{\kappa_1}^{\kappa_0} f(\kappa) \;d\kappa -1 + \delta \bigg) 
\end{align*}
The partial derivatives are 
\begin{align*}
    \frac{\partial L}{\partial f(\kappa)} &= - \log (f(\kappa)) - 1 - \lambda_0 - \lambda_1 \ind(\kappa > \kappa_1) \\
    \frac{\partial L}{\partial \lambda_0} &= - \int_0^{\kappa_0} f(\kappa) \;d\kappa + 1 \\
    \frac{\partial L}{\partial \lambda_1} &= \int_{\kappa_1}^{\kappa_0} f(\kappa) \;d\kappa -1 + \delta
\end{align*}
giving 
\begin{align*}
    f(\kappa) = \exp{(- \lambda_0 - 1 - \lambda_1 \ind(\kappa > \kappa_1))}
\end{align*}
with conditions, 
\begin{align*}
    \int_0^{\kappa_0} f(\kappa) \;d\kappa &= 1 \\
    \int_{\kappa_1}^{\kappa_0} f(\kappa) \;d\kappa &= 1 - \delta.
\end{align*}
Solving for $\lambda_0$ and $\lambda_1$ gives 
\begin{align*}
    \exp(-\lambda_0 -1) &= \frac{\delta}{\kappa_1} \\
    \exp(-\lambda_1) &= \bigg(\frac{1 - \delta}{\delta}\bigg) \bigg(\frac{\kappa_1}{\kappa_0 - \kappa_1}\bigg) \\  \\
\end{align*}
and 
\begin{align*}
    f(\kappa) = 
    \begin{cases}
        \frac{\delta}{\kappa_1} &\quad \text{if } t \le \kappa_1 \\
        \frac{1 - \delta}{\kappa_0 - \kappa_1} &\quad \text{otherwise}
    \end{cases}
\end{align*}
and 
\begin{align*}
    F(\kappa) = 
    \begin{cases}
        \delta\frac{t}{\kappa_1} &\quad \text{if } t \le \kappa_1 \\
        \delta + (1 - \delta)\frac{t - \kappa_1}{\kappa_0 - \kappa_1} &\quad \text{otherwise}.
    \end{cases}
\end{align*}

\section{Extending Our Theory to Solution Quality}\label{app:solutionquality}

The extension to include solutions of differing quality is fairly straightforward. We can mean anything we want by ``quality'', as long as we can assign it a numerical value (e.g., mean-squared error on a machine learning model, fuel saved on our delivery route, our subjective rating of the beauty of a generated image, etc.). For our purposes, all that matters is that algorithms are now distributions over \emph{pairs} of numbers $(t, q)$ where $t$ is the runtime and $q \in [q_{-1}, q_1]$ is the quality of the solution returned. We have some default action available to us that gives a solution with quality $q_0$ (maybe this is the loss on the model with the default parameter setting, or our impression of a random white-noise image). This means that we can effectively constrain $q$ to the interval $[q_0, q_1]$, where the worst-quality solution $q_0$ is always available to us. We assume that $q_1 > q_0$.

The statement of the first three axioms is unchanged. The statements of Independence, Eagerness and Relevance do change, but each retains its fundamental interpretation and ultimately plays the same role in the theorem's proof.

\paragraph{Axiom 1 \normalfont{(Transitivity)}.} If $A \prefgteq{K} B$ and $B \prefgteq{K} C$, then $A \prefgteq{K} C$.
\paragraph{Axiom 2 \normalfont{(Monotonicity)}.} If $A \prefgteq{K} B$ then for any $p,q \in [0,1]$ we have ${[p:A , (1-p):B]} \prefgteq{K} {[q:A , (1 - q):B]}$ if and only if $p \ge q$.
\paragraph{Axiom 3 \normalfont{(Continuity)}.} If $A \prefgteq{K} B \prefgteq{K} C$, then there exists a $p \in [0,1]$ such that $B \prefeq{K} {[p:A , (1-p):C]}$.
\paragraph{Axiom 4$^\prime$ \normalfont{(Independence)}.} If $\delta_t \times \delta_q \prefeq{\delta_{\kappa}} M(t, q, \kappa)$ for all $t, q, \kappa$, then $A \prefeq{K} [M(t, q, \kappa) \st (t,q) \sim A, \kappa \sim K]$.
\paragraph{Axiom 5$^\prime$ \normalfont{(Eagerness)}.} For any $t \le t'$ and $q \ge q'$, if the support of $A$ is contained in $[t, t'] \times [q', q]$, then $\delta_t \times \delta_q \prefgteq{K} A \prefgteq{K} \delta_{t'} \times \delta_{q'}$.
\paragraph{Axiom 6$^\prime$ \normalfont{(Relevance)}.} $\delta_t \times \delta_q \prefgt{\delta_{\kappa}} \delta_t \times \delta_{q_0}$ for all $t < \kappa$ and all $q > q_0$.

The function $p$ now takes three arguments, but its interpretation as the ``balance point'' between the best and worst possible outcomes remains the same. 

\begin{definition}\label{def:pquality}
    Set $p(t, q, \kappa) = 0$ if $t \ge \kappa$, and otherwise set $p(t, q, \kappa)$ to be the value that satisfies
    \begin{align}\label{eqn:pf:thm:quality:p}
        \delta_t \times \delta_q \prefeq{\delta_{\kappa}} \Big[ p(t, q, \kappa) :  \delta_0 \times \delta_{q_1} \,,\, \big(1 - p(t, q, \kappa) \big) : \delta_{\kappa} \times \delta_{q_0} \Big].
    \end{align}
Since Eagerness tells us that $\delta_{0} 
\times \delta_{q_1} \prefgteq{\delta_{\kappa}} \delta_{t} \times \delta_{q} \prefgteq{\delta_{\kappa}} \delta_{\kappa} \times \delta_{q_0} $ when $t < \kappa$, Continuity ensures that $p$ exists for any $t < \kappa$ and any $q$, and Monotonicity ensures it is unique. So $p$ is defined for all $t,q$ and $\kappa$. 
\end{definition}

The main theorem of \cref{sec:axioms} (\cref{thm:main}) can now be restated to include solution quality. The function $p$ is monotonically increasing in quality $q$, and for any fixed $q > q_0$ it behaves just as it did in \cref{thm:main}.

\begin{theorem}\label{thm:quality}
    If our preferences follow the axioms as stated in this section, then a function $u$ satisfies 
    \begin{align}\label{eqn:thm:quality:uutil}
        A \;\prefgteq{K}\; B  \iff \E_{(t,q) \sim A, \kappa \sim K}\big[u(t, q, \kappa)\big] \ge \E_{(t,q) \sim B, \kappa \sim K}\big[u(t, q, \kappa)\big],
    \end{align}
for any runtime distributions $A$ and $B$ and any timeout distribution $K$ if and only if there are constants $c_0$ and $c_1 > 0$ such that $u(t, q, \kappa) = c_1 p(t, q, \kappa) + c_0$. Furthermore, $p$ has the form
\begin{enumerate}
    \item $p(0, q_1, \kappa) = 1$ (maximum achieved at $t=0$ and $q = q_1$),
    \item $p(t, q, \kappa) \ge p(t', q, \kappa)$ for all $t \le t'$ and any $q$ (monotonically decreasing in $t$),
    \item $p(t, q, \kappa) \ge p(t, q', \kappa)$ for all $q \ge q'$ and any $t$ (monotonically increasing in $q$), 
    \item $p(t, q, \kappa) > 0$ for all $t < \kappa$, $q > q_0$ (strictly positive if we improve $q$),
    \item $p(\kappa, q, \kappa) = 0$ for any $q$ (minimum always achieved at $t=\kappa$).
\end{enumerate}
\end{theorem}

\begin{proof}
Given an arbitrary runtime distribution $A$ and a timeout distribution $K$, we will construct a new synthetic algorithm $X$ that returns an answer either instantaneously or after some amount of time sampled from $K$. Formally, 
\begin{align}\label{eqn:pf:thm:quality:x1}
    X = \bigg[ \Big[ p(t, q, \kappa) :  \delta_0 \,,\, \big(1 - p(t, q, \kappa) \big) : \delta_{\kappa} \Big] \;\big|\; (t,q) \sim A, \kappa \sim K \bigg] .
\end{align}
Where $p$ is defined in \cref{def:pquality}. Setting 
\begin{align*}
    p_A &= \int_{\kappa} \int_{t, q} p(t, q, \kappa) dF_A(t, q) dF_{K}(\kappa) \\
    &= \E_{(t, q) \sim A, \kappa \sim K}\big[ p(t, q, \kappa) \big],
\end{align*}
we can write $X$'s runtime  distribution as \begin{align}\label{eqn:pf:thm:quality:x2}
    X \;=\; \big[ p_A : \delta_0 \,,\, (1 - p_A) : K \big].
\end{align}
Consider the function $M(t, q, \kappa) = [p(t, q, \kappa) :  \delta_0 \times \delta_{q_1} \,,\, (1 - p(t, q, \kappa) ) : \delta_{\kappa} \times \delta_{q_0}]$ that maps runtime--quality--captime triplets to mixture distributions. Since $p$ was defined in \cref{eqn:pf:thm:quality:p} so that $\delta_{t} \times \delta_{q} \prefeq{\delta_{\kappa}} M(t, q, \kappa)$, we can conclude from Independence that 
\begin{align}\label{eqn:pf:thm:quality:akx}
    A \prefeq{K} \bigg[ \Big[ p(t, q, \kappa) :  \delta_0 \times \delta_{q_1} \,,\, \big(1 - p(t, q, \kappa) \big) : \delta_{\kappa} \times \delta_{q_0} \Big] \;\big|\; (t, q) \sim A, \kappa \sim K \bigg].
\end{align}
\cref{eqn:pf:thm:quality:x1,eqn:pf:thm:quality:x2,eqn:pf:thm:quality:akx} together then give that 
    \begin{align}
        A \prefeq{K} \big[ p_A : \delta_0 \times \delta_{q_1} \,,\, (1 - p_A) : K \big].
    \end{align}
Now consider a second algorithm $B$, and define $Y$ and $p_B$ analogously to $X$ and $p_A$, but with $B$ in place of $A$, so that by the same argument we have
    \begin{align}
        B \prefeq{K} \big[ p_B : \delta_0 \times \delta_{q_1} \,,\, (1 - p_B) : K \times \delta_{q_0} \big].
    \end{align}
Since $\delta_0 \times \delta_{q_1} \prefgteq{K} K \times \delta_{q_0}$ by Eagerness, Monotonicity tells us that $\big[ p_A : \delta_0 \times \delta_{q_1} \,,\, (1 - p_A) : K \times \delta_{q_0} \big] \prefgteq{K} \big[ p_B : \delta_0 \times \delta_{q_1} \,,\, (1 - p_B) : K \times \delta_{q_0} \big]$ iff $p_A \ge p_B$, and thus
\begin{align}\label{eqn:pf:thm:quality:putil}
    A \prefgteq{K} B \iff \E_{(t, q) \sim A, \kappa \sim K} \big[ p(t, q, \kappa) \big] \ge \E_{(t, q) \sim B, \kappa \sim K} \big[ p(t, q, \kappa) \big] .
\end{align}
So the function $p$ can serve as a utility function, and we can use the biconditional \cref{eqn:pf:thm:quality:putil} to infer certain aspects of $p$'s form: 

\begin{enumerate}
    \item By definition of $p$ we have $\delta_0 \times \delta_{q_1} \prefeq{\delta_{\kappa}} \big[ p(0, q_1, \kappa) : \delta_0 \times \delta_{q_1} \,,\, (1 - p(0, q_1, \kappa)) : \delta_{\kappa} \times \delta_{q_0} \big]$, where $p(0, q_1, \kappa) \le 1$, and by Eagerness we have $\delta_0 \times \delta_{q_1} \prefgteq{\delta_{\kappa}} \delta_{\kappa} \times \delta_{q_0}$ so applying Monotonicity with $A=\delta_0 \times \delta_{q_1}, B=\delta_{\kappa} \times \delta_{q_0}$ and $q = 1$, we have that $p(0, q_1, \kappa) \ge 1$, and thus $p(0, q_1, \kappa) = 1$. 
    \item For any $t \le t' < \kappa$ and any $q$, Eagerness tells us that $\delta_t \times \delta_{q} \prefgteq{\delta_{\kappa}} \delta_{t'} \times \delta_{q}$, and so $p(t, q, \kappa) \ge p(t', q, \kappa)$.
    \item For any $q \ge q'$ and any $t < \kappa$, Eagerness tells us that $\delta_t \times \delta_{q} \prefgteq{\delta_{\kappa}} \delta_{t} \times \delta_{q'}$, and so $p(t, q, \kappa) \ge p(t, q', \kappa)$.
    \item Relevance states that $\delta_t \times \delta_q \prefgt{\delta_{\kappa}} \delta_t \times \delta_{q_0}$ for all $t < \kappa$ and all $q > q_0$, and Eagerness says that $\delta_t \times \delta_{q_0} \prefgt{\delta_{\kappa}} \delta_{\kappa} \times \delta_{q_0}$, and so $p(t, q, \kappa) > p(\kappa, q, \kappa) = 0$.
    \item By definition, $p(\kappa, q, \kappa)$ is set to $0$ for all $q$.
\end{enumerate}

So $p$ has the given form and can serve as a utility function. We can now show that a function $u$ satisfies \cref{eqn:thm:quality:uutil} if and only if it has the form $u(t, q, \kappa) = c_1 p(t, q, \kappa) + c_0$ for some $c_1> 0$ and $c_0$. The reverse, `only if' direction follows immediately from linearity of expectation. For the forward, `if' direction, suppose that $u$ does satisfy \cref{eqn:thm:quality:uutil} for all $A, B$ and $K$. Since $\delta_0 \times \delta_{q_1} \prefgteq{K} A \prefgteq{K} \delta_\infty \times \delta_{q_0}$ by Eagerness, Continuity says that there exists a constant $\alpha$ such that $A \prefeq{K} \big[ \alpha : \delta_0 \times \delta_{q_1} , (1 - \alpha) : \delta_{\infty} \times \delta_{q_0} \big]$. Using this equivalence, \cref{eqn:pf:thm:quality:putil} tells us that
\begin{align*}
    \E_{(t,q) \sim A, \kappa \sim K} \big[ p(t, q, \kappa) \big] = \alpha,
\end{align*}
and applying \cref{eqn:thm:main:uutil} to $A$ and the equivalent mixture $\big[ \alpha : \delta_0 \times \delta_{q_1}, (1 - \alpha) : \delta_{\infty} \times \delta_{q_0}\big]$ tells us that 
\begin{align*}
    \E_{(t, q) \sim A, \kappa \sim K} \big[ u(t, q, \kappa) \big] &= \alpha \E_{\kappa \sim K} \big[ u(0, q_1, \kappa) \big] + (1 - \alpha) \E_{\kappa \sim K} \big[ u(\infty, q_0, \kappa) \big] \\
    &= \alpha \big( \E_{\kappa \sim K} \big[ u(0, q_1, \kappa) - u(\infty, q_0, \kappa) \big] \big) + \E_{\kappa \sim K} \big[ u(\infty, q_0, \kappa) \big],
\end{align*}
so setting $c_1= \E_{\kappa \sim K} \big[ u(0, q_1, \kappa) - u(\infty, q_0 \kappa) \big]$ and $c_0 = \E_{\kappa \sim K} \big[ u(\infty, q_0, \kappa) \big]$ we have that
\begin{align*}
    \E_{(t,q) \sim A, \kappa \sim K} \big[ u(t, q, \kappa) \big] = c_1 \E_{(t,q) \sim A, \kappa \sim K} \big[ p(t, q, \kappa) \big] + c_0
\end{align*}
for any $A$ and $K$. In particular, when $A = \delta_t$ and $K = \delta_{\kappa}$ for arbitrary $t$ and $\kappa$ we get that $u(t, q, \kappa) = c_1 p(t, q, \kappa) + c_0$, which completes the proof.
\end{proof}

%%%%%%%%%%%%%%%%%%%%%%%%%%%%%%%%%%%%%%%%%%%%%%%%%%%%%%%%%%%%%%%%%%%%%%%%%%%%%%%
%%%%%%%%%%%%%%%%%%%%%%%%%%%%%%%%%%%%%%%%%%%%%%%%%%%%%%%%%%%%%%%%%%%%%%%%%%%%%%%

\end{document}

%% file: summary_table.tex
\begin{figure}    % \renewcommand\theadfont{\large}
    \renewcommand\theadfont{}
    \renewcommand\theadgape{}
    \renewcommand\cellgape{}
    \newcommand{\tblspc}{\addlinespace[.9em]}
    \centering
    \scriptsize
    \begin{tabular}{ c c c }
    \toprule
    \multicolumn{3}{c}{\textit{\textbf{\thead{``The best algorithm...}}}} \\
    %\hline
    & \hspace{1em} \makecell{constant utility} \hspace{1em} & \hspace{1em} \makecell{general utility} \hspace{1em} \\ 
    \midrule
    \makecell{known\\captime:\\$\kappa = \kappa_0$} & \makecell{\textit{\textbf{... solves the most instances.''}}\\(\cref{ex:known-constant})} & \makecell{\textit{\textbf{... gives the greatest expected}}\\ \textit{\textbf{proportion of benefit.''}}\\(\cref{ex:known-general})} \\ 
    \tblspc
    \makecell{known\\distribution:\\$\kappa \sim K$} & \makecell{\textit{\textbf{... is most likely to solve}}\\ \textit{\textbf{an instance from $K$.''}}\\(\cref{ex:unknown-constant})} & \makecell{\textit{\textbf{... gives the greatest expected}}\\ \textit{\textbf{proportion of benefit, in}}\\ \textit{\textbf{expectation over $K$.''}}\\(numerical methods)} \\ 
    \tblspc
    \makecell{unknown\\distribution:\\$\kappa \sim ?$} & \makecell{\textit{\textbf{... is most likely to solve an}}\\ \textit{\textbf{instance from the Maximum}}\\ \textit{\textbf{Entropy distribution.''}}\\(\cref{sec:maxent})} & \makecell{\textit{\textbf{... gives the greatest expected}}\\ \textit{\textbf{proportion of benefit, in}}\\ \textit{\textbf{expectation over the}}\\ \textit{\textbf{Maximum Entropy}}\\ \textit{\textbf{distribution.''}}\\(numerical methods)} \\ 
    \bottomrule
    \end{tabular}
    \vspace{-.15cm}
    \caption{Implied scoring functions for different scenarios.}\label{fig:summary_table}
\end{figure}